\documentclass[10pt,letterpaper]{article}
\usepackage[letterpaper,margin=1in]{geometry}
\usepackage[T1]{fontenc}
\usepackage{times}
\usepackage{microtype}
\usepackage{graphicx}
\usepackage[round,authoryear]{natbib}
\usepackage{xcolor}
\usepackage{amsmath,amssymb,amsthm,mathtools}
\usepackage{algorithm,algorithmic}
\usepackage{booktabs}
\usepackage[colorlinks,
            linkcolor=blue,
            anchorcolor=blue,
            citecolor=teal,
            ]{hyperref} 
\usepackage{url}
\usepackage{enumitem}
\usepackage{comment}
\usepackage{subcaption}

\newtheorem{theorem}{Theorem}
\newtheorem{lemma}{Lemma}
\newtheorem{proposition}{Proposition}
\newtheorem{corollary}{Corollary}

\newcommand{\E}{\mathbb{E}}
\newcommand{\Var}{\mathrm{Var}}
\newcommand{\Cov}{\mathrm{Cov}}

\newcommand{\pip}{\pi_\alpha}
\newcommand{\rself}{r_{\mathrm{self}}}
\newcommand{\rstar}{r^{\star}}
\newcommand{\calY}{\mathcal{Y}}
\newcommand{\deltasg}{\delta}
\newcommand{\sigCsq}{\sigma^2_C}
\newcommand{\sigIsq}{\sigma^2_I}
\newcommand{\muC}{\mu_C}
\newcommand{\muI}{\mu_I}
\newcommand{\Deltas}{\Delta_s}

\title{Adaptive Power Sampling for LLM Reasoning }

\newcommand{\authorblock}[3]{%
  \begin{tabular}[t]{@{}c@{}}
    \normalsize #1 \\
    \small #2 \\
    \footnotesize\texttt{#3}
  \end{tabular}%
}
\author{
  \begin{tabular}{@{}c@{\hspace{1.5em}}c@{\hspace{1.5em}}c@{}}
    \authorblock{Bingnan Xiao\thanks{These authors contributed equally to this work.}}{Fudan University}{22110720061@m.fudan.edu.cn} &
    \authorblock{Chenhao Yang\footnotemark[1]}{University of Glasgow}{3165453Y@student.gla.ac.uk} &
    \authorblock{Bingcong Li\thanks{Corresponding authors: Bingcong Li (\texttt{bingtsongli@gmail.com}) and Xin Wang (\texttt{xwang11@fudan.edu.cn}).}}{ETH Zurich}{bingtsongli@gmail.com} \\[1.5ex]
    \multicolumn{3}{c}{%
      \begin{tabular}{@{}c@{\hspace{3em}}c@{}}
        \authorblock{Wei Ni}{Edith Cowan University}{wei.ni@ieee.org} &
        \authorblock{Xin Wang\footnotemark[2]}{Fudan University}{xwang11@fudan.edu.cn}
      \end{tabular}%
    }
  \end{tabular}
}

\date{}
\begin{document}
\maketitle

\begin{abstract}

Sequence-level power sampling has recently emerged as a training-free approach to reasoning by sampling from a sharpened output distribution of a base large language model (LLM). Nevertheless, existing methods typically sharpen the base model distribution uniformly across queries, overlooking variations in query difficulty and in how well the base model already handles each query. The goal of this work is to equip power sampling with query adaptivity. Theoretically, we show that the benefits of further sharpening are determined by the self-reward gap between correct and incorrect responses. Based on this insight, we propose \emph{Adaptive Power Sampling} (APS), which adjusts the sharpening exponent on a per-query basis at test time using the relationship between answer agreement and the model's self-reward. Experiments across diverse reasoning tasks, including MATH500, HumanEval, and GPQA, show that APS consistently outperforms power sampling with a fixed sharpening exponent, without additional training.

\end{abstract}

\section{Introduction}
\label{sec:introduction}

Reasoning is fundamentally about generating and identifying high-quality solution trajectories from a model's output distribution. Recent work suggests that post-training often achieves this by redistributing probability mass toward sequences that are more likely to yield correct answers~\citep{huang2024sharpening,yue2025rlvr,he2025rewarding}. 
This perspective, commonly referred to as \emph{sequence-level distribution sharpening}, also provides a foundation for test-time reasoning~\citep{karan2025reasoning}. In particular, strong reasoning performance on tasks such as math and coding can be achieved by sampling from a sharpened version of the base model's sequence-level distribution, even when the base model has not undergone post-training~\citep{karan2025reasoning,ji2026scalable}.

Compared with post-training, a clear advantage of test-time reasoning is its reduced reliance on post-training data. This benefit, however, comes at the cost of query-adaptive sharpening. Post-training can implicitly learn, \textit{for each query or prompt}, an appropriate degree of sharpening, concentrating sufficient probability mass on high-quality responses while preserving enough diversity for the target task. In contrast, test-time reasoning methods typically commit to a single predetermined degree of sharpening, ``betting on'' how concentrated the desired distribution should be
\citep{karan2025reasoning,azizi2026powersmc,ji2026scalable,nguyen2026future,meng2026degs}.

Such a fixed choice, however, may not be sufficient. Intuitively, the appropriate degree of sharpening should vary across problems and domains, e.g., between math and coding, or also within the same domain, between, e.g., primary-school and college-level math problems. Motivated by this intuition, which we further validate experimentally, we study whether the sharpening strength can instead be determined \textit{query-wise} at test time, without relying on post-training data or manual tuning.
To equip power sharpening with adaptivity, we first study how reasoning performance varies with the sharpening strength. More precisely,
let $x$ denote a prompt, and let $p_\theta(y\mid x)$ denote the probability that a base model $\theta$ generates a complete response $y$ given $x$. 
Test-time sharpening \citep{karan2025reasoning} samples responses from the sharpened distribution
$ \pi_\alpha(y\mid x) \propto p_\theta(y\mid x)^\alpha$,
where $\alpha \geq 1$ controls the degree of distributional sharpening. 
We show that for this particular query $x$, the preferred sharpening direction is jointly determined by the relationship between response correctness and sequence log-likelihood.
Stronger sharpening, i.e., larger $\alpha$, is more beneficial when correct responses tend to have higher sequence log-likelihood than incorrect ones, and harmful otherwise. 
Translating these theoretical insights into practice, we propose \emph{Adaptive Power Sampling} (APS), a training-free reasoning method that dynamically adjusts the sharpening exponent $\alpha$ at test time. Since ground-truth correctness is unavailable, APS leverages self-consistency \citep{wang2023self,aggarwal2023adaptiveconsistency} among sampled answers as a surrogate signal. Combined with the theoretical characterization above, APS adjusts $\alpha$ on the fly for each query. Such adaptivity improves reasoning performance across several tasks. In a nutshell, our main contributions are summarized as follows:
\begin{itemize}
    \item Motivated by our empirical observation that the preferred sharpening strength varies across queries, 
    we develop a theoretical characterization of how reasoning performance changes with the sharpening strength. For each query, whether a stronger sharpening is preferred is determined by the self-reward gap between correct and incorrect responses.

    \item Building on these insights, we propose \emph{Adaptive Power Sampling} (APS), which adjusts the sharpening exponent $\alpha$ for each query at test time using statistics from sampled responses. APS achieves query-wise adaptivity analogous to that of RL-based post-training, while requiring no additional training or post-training data.

    \item We evaluate APS on multiple models and reasoning benchmarks, showing consistent improvements over fixed-exponent power sampling, and competitive
    performance against RL-post-trained methods.
\end{itemize}

\section{Related Work}
\label{sec:related}

\textbf{Reasoning via post-training.}
Reinforcement learning (RL) has been shown to improve the reasoning performance of LLMs on mathematical and coding tasks \citep{ouyang2022training,shao2024deepseekmath,guo2025deepseekr1,
yu2025dapo,liu2025understanding,lin2025cppo}. Recent work \citep{huang2024sharpening,yue2025rlvr,he2025rewarding,shao2025spurious,song2025outcome} suggests that RL post-training can reinforce existing
high-probability responses in the model, highlighting
distribution reshaping as an important mechanism behind reasoning performance improvement. The view motivates investigating whether similar gains can be achieved by directly sharpening the model distribution and adapting the sharpening strength at inference time.

\textbf{Inference-time scaling} improves LLM performance by allocating additional computation at test time~\citep{zhang2025survey,hariri2026testtime}. Representative approaches include repeated sampling and answer aggregation or selection ~\citep{wang2023self,li2022competition,brown2024large}, verifier- or reward-guided selection of candidate solutions~\citep{cobbe2021training,lightman2023lets,chow2024inference}, and explicit search over intermediate reasoning states~\citep{xie2023self,yao2023tree,zhou2024language,snell2024scaling}. While effective, these methods motivate more direct control of the model's response
distribution at inference time.

\textbf{Reasoning at test time.}
Power sampling directly sharpens the model's sequence-level distribution by assigning greater weight to high-probability responses~\citep{karan2025reasoning}. It provides a training-free alternative that can achieve reasoning performance comparable to RL post-training while better preserving generation diversity.  Recent studies have further improved the efficiency of sequence-level power sampling via,
e.g., a future-aware autoregressive approximation~\citep{ji2026scalable}, particle and sequential Monte Carlo methods~\citep{azizi2026powersmc,nguyen2026future,wang2026cache} and entropy-guided proposals~\citep{guo2026egps,zhou2026cutting}. While these methods improve sampling efficiency, they typically rely on a predetermined sharpening strength across queries. In contrast, our work studies how the sharpening strength should be adaptive to individual queries.

Power sampling can also be understood from a theoretical perspective. \citep{tomihari2026bridges} connects power sampling with self-reward KL-regularized RL and self-distillation. In parallel, our work studies how reasoning performance changes with the sharpening strength, 
thereby motivating the adaptive selection of sharpening strength at test time.
\section{Preliminaries}
\label{sec:setup}
\subsection{Sequence-Level Power Distribution}
Given prompt $x$, let $p_\theta(y|x)$ denote the probability of a base model $\theta$ generating a complete response $y=(y^1,\ldots,y^T)$. In this work, $p_\theta(y|x)$ is referred to as the sequence-level likelihood, and by the autoregressive factorization, it can be equivalently obtained via
$p_\theta(y|x) = \prod_{t=1}^{T}p_\theta(y^t|x,y^{<t})$.

Power sampling, originally proposed in \citep{karan2025reasoning}, leverages a sharpened sequence-level distribution to re-weight the base model's probability by a sharpening exponent $\alpha$. For $\alpha\ge 1$, the sequence-level power distribution is defined as
\begin{equation} \label{eq:power_distribution}
\pi_\alpha(y|x) = \frac{p_\theta(y|x)^\alpha}{Z_\alpha(x)},
\end{equation}
where $Z_\alpha(x) = \sum_{y}p_\theta(y|x)^\alpha$ is the normalization constant. The exponent $\alpha=1$ recovers the base model distribution, while a larger $\alpha$ concentrates probability mass on higher-likelihood complete responses.  To illustrate the sharpening effect, consider two complete responses $y$ and
$y'$ with $p_\theta(y\mid x)>p_\theta(y'\mid x)$. For $\alpha>1$, their
relative probabilities satisfy
\begin{equation}
\frac{\pi_\alpha(y\mid x)}{\pi_\alpha(y'\mid x)}
=
\left(
\frac{p_\theta(y\mid x)}{p_\theta(y'\mid x)}
\right)^\alpha
>
\frac{p_\theta(y\mid x)}{p_\theta(y'\mid x)}.
\end{equation}
Thus, increasing $\alpha$ further amplifies the relative advantage of responses with higher base-model likelihood, making the sequence-level distribution sharper. Note that \textit{sequence}-level sharpening is fundamentally different from \textit{token}-level low-temperature scaling in autoregressive decoding \citep{wang2020contextual}: the former reweights complete
sequences according to their sequence likelihoods, whereas the latter rescales the
next-token distribution locally at each decoding step. This has been demonstrated multiple times, e.g., \citep{karan2025reasoning,azizi2026powersmc,ji2026scalable,nguyen2026future,meng2026degs}. 
\subsection{One Fixed Exponent Does Not Fit Every Query } \label{motiexample}

Existing power sampling algorithms~\citep{karan2025reasoning,azizi2026powersmc,ji2026scalable,nguyen2026future,meng2026degs} rely on a fixed sharpening exponent $\alpha$. While it is natural to expect different tasks, such as coding and mathematics, to favor different values of $\alpha$, in this subsection we show that even within the same dataset, \textit{no single $\alpha$ performs equally well across all queries}.

To illustrate this, we evaluate Qwen2.5-Math-7B~\citep{yang2024qwen25math} on 20 randomly sampled queries from the MATH500 dataset~\citep{lightman2023lets}. We consider power-sampling exponents $\alpha\in \{1,2,4,8,16 \}$. For each query and each value of $\alpha$, we generate $N=16$ responses per random seed using the MCMC sampler of~\citet{karan2025reasoning}.
We use $P(\alpha\mid x)$ to denote the fraction of correct responses for query $x$ at exponent $\alpha$. 

 \begin{figure}[t]
  \centering
  \includegraphics[width=0.9\linewidth]{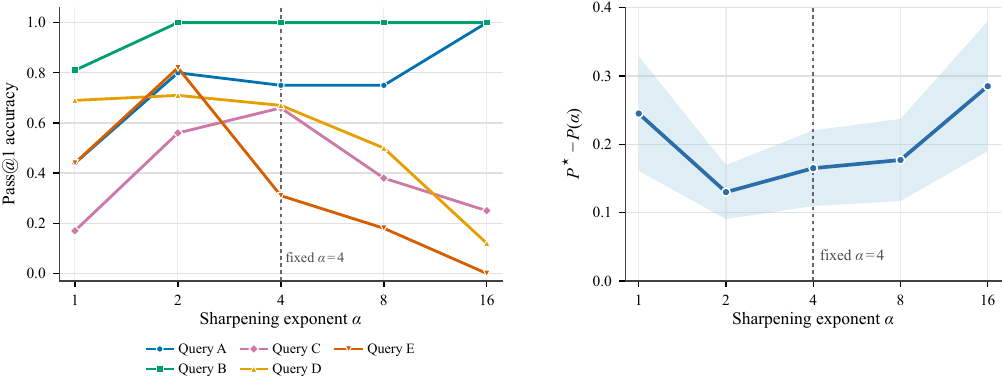}
  \caption{
  Motivating example on MATH500 with Qwen2.5-Math-7B.
  The left panel shows representative Pass@1 accuracy curves under power sampling over
  $\alpha\in\{1,2,4,8,16\}$.
  The right panel shows the mean oracle gap $P^\star-P(\alpha)$.
  }
  \label{fig:motivation}
\end{figure}

The results are shown in Figure~\ref{fig:motivation}.
In the left panel, 5 representative queries are plotted to visualize query-wise optimal $\alpha$. 
In particular, Query~A attains its highest pass@1 accuracy only at $\alpha=16$ while Query~B already reaches perfect accuracy at $\alpha=2$. In contrast, Queries~C--E achieve their best performance at an intermediate exponent.
In the right panel, we measure the performance loss from using a common exponent instead of the best exponent for each of the 20 sampled queries. For each query $x$, let $P^{*}(x)=\max_{a} P(a\mid x)$ denote its best accuracy over the tested exponents, and define the oracle gap as $P^{*}(x)-P(\alpha\mid x)$. The gap is smallest at $\alpha=2$, with an average value of $0.13$, whereas the fixed choice $\alpha=4$ used in \citep{karan2025reasoning} gives a larger gap of $0.17$. These results jointly show that a single fixed $\alpha$ can leave a noticeable performance gap relative to the best exponent for individual queries, motivating an adaptive choice of $\alpha$   at test time.

\section{Effect of the Sharpening Strength}
\label{sec:aap}
In this section, we theoretically study how $\alpha$ affects
reasoning performance under the sequence-level power distribution.
We characterize its effect on test-time correctness and identify when sharpening is beneficial or harmful. Following the same analytical framework, we also obtain an extension to group-relative RL as a byproduct, where adaptive sharpening can likewise be beneficial.
Throughout this section, we consider a fixed query $x$ and omit the dependence on $x$ for notational simplicity. 
\subsection{Definitions}

To examine how the sharpening exponent $\alpha$ affects the reasoning performance, we restrict our discussion to tasks with binary correctness outcomes, such as mathematics and coding. 
For a given prompt $x$, let $\rstar(y)\in\{0,1\}$ denote the binary task reward
of a complete response $y$, where $\rstar(y)=1$ indicates a correct response and $\rstar(y)=0$ indicates an incorrect response. Accordingly, we partition the response space into correct and incorrect sets, i.e., 
\begin{equation}
    \calY_C=\{y:\rstar(y)=1\},
    \quad
    \calY_I=\{y:\rstar(y)=0\}.
\end{equation}
The expected performance under exact sampling from $\pip$ is defined as
\begin{equation}
    P(\alpha)=\E_{\pip}[\rstar(Y)].
    \label{eq:performance}
\end{equation}
Since $\rstar$ is binary, $P(\alpha)$ is also the probability that a response sampled from $\pip$ is correct. Hence, a preferred sharpening exponent enlarges $P(\alpha)$.
Another useful quantity is the model's self-reward, which reflects the sequence-level likelihood assigned to a response $y$
\begin{equation}
    \rself(y)=\log p_\theta(y|x).
    \label{eq:self_reward}
\end{equation}
Under $\pip$, the conditional self-reward means and variances with respect to $\alpha$ is defined as
\begin{align}
    \muC(\alpha)&=\E_{\pip}[\rself(Y)\mid \rstar(Y)=1],
    & \muI(\alpha)&=\E_{\pip}[\rself(Y)\mid \rstar(Y)=0],\\
    \sigCsq(\alpha)&=\Var_{\pip}[\rself(Y)\mid \rstar(Y)=1],
    & \sigIsq(\alpha)&=\Var_{\pip}[\rself(Y)\mid \rstar(Y)=0].
\end{align}
We then define the self-reward gap $\deltasg(\alpha)$ to measure whether correct responses have higher model log-likelihood than incorrect responses under the current power distribution
\begin{equation}
    \deltasg(\alpha)=\muC(\alpha)-\muI(\alpha).
    \label{eq:self_reward_gap}
\end{equation}

\subsection{Theoretical Basis for Adaptive Sharpening}
\label{sec4.2}

For the test-time objective $P(\alpha)$, we study how the expected reasoning performance changes as $\alpha$ increases.
From \eqref{eq:power_distribution}, increasing $\alpha$ shifts relative probability mass toward responses with higher base-model likelihood.
Consequently, the marginal effect of sharpening on $P(\alpha)$ depends on whether model likelihood is aligned with the task reward.
The following theorem makes this relation exact.
\begin{theorem}
\label{thm:sharpening_derivative}
For the fixed prompt $x$ and the reward $\rstar$ independent of $\alpha$, we have 
\begin{equation}
    \frac{dP(\alpha)}{d\alpha}
    =
    \Cov_{\pip}(\rstar(Y),\rself(Y)).
    \label{eq:sharpening_derivative}
\end{equation}
\end{theorem}
Theorem~\ref{thm:sharpening_derivative} provides an exact characterization of the marginal effect of sharpening. If task reward and self-reward are positively correlated under the current power distribution, increasing $\alpha$ improves exact expected performance. 
Conversely, if the covariance is negative, additional sharpening moves probability mass in a direction that hurts the task reward. Interestingly, the covariance in
\eqref{eq:sharpening_derivative} can be further decomposed into a simple self-reward gap, as given below.

\begin{corollary}[Self-reward gap decomposition]
\label{cor:gap_decomposition}
Based on the task reward $\rstar\in\{0,1\}$, we obtain
\begin{equation}
    \frac{dP(\alpha)}{d\alpha}
    =
    P(\alpha)(1-P(\alpha))\deltasg(\alpha).
    \label{eq:gap_decomposition}
\end{equation}
\end{corollary}
Corollary 1 has a direct interpretation. Here, $P(\alpha)(1-P(\alpha))$ measures how much room is left for the correctness probability to change under the current power distribution: it is largest when correct and incorrect responses are both likely, i.e., $P(\alpha)=0.5$, and it becomes small when the sampler almost always succeeds or almost always fails. 
Prior work connects power sampling to the correlation between task reward and
self-reward ~\citep{tomihari2026bridges}. Theorem~\ref{thm:sharpening_derivative} and Corollary \ref{cor:gap_decomposition} further characterize the local change in expected
performance caused by increasing $\alpha$. Since $P(\alpha) \in [0,1]$ for binary $\rstar$,
increasing $\alpha$ locally improves performance if $\delta(\alpha)>0$ and locally degrades performance if $\delta(\alpha)<0$. 
This reveals the source of adaptivity: if $\delta(\alpha)$ is known, we can determine whether increasing or decreasing $\alpha$ would improve performance for this particular query.
Next, we study how
$\delta(\alpha)$ evolves with $\alpha$ to understand when further
sharpening may switch from improving performance to degrading it, or vice
versa.

\begin{lemma}[Self-reward gap dynamics]
\label{lem:gap_dynamics}
The self-reward gap evolves as
\begin{equation}
    \frac{d\deltasg(\alpha)}{d\alpha}
    =
    \sigCsq(\alpha)-\sigIsq(\alpha).
    \label{eq:gap_dynamics}
\end{equation}
\end{lemma}

The self-reward gap increases when correct sequences have higher
conditional log-probability variance than incorrect sequences, and
decreases under the reverse ordering. Since the sign of
$dP(\alpha)/d\alpha$ is determined by $\delta(\alpha)$ 
for binary $\rstar$,
a larger conditional variance in the correct class
moves the gap toward more positive values and thereby shifts the local effect of sharpening toward improvement. 
Besides, we can differentiate the first-order identity, yielding the second-order characterization of $d^2P/d\alpha^2$ in Appendix \ref{more results}.

While extending the above local optimality analysis to a characterization of the globally optimal $\alpha^* \in [1, \infty)$ is challenging, we can still gain insight into how $P(\alpha)$ evolves with $\alpha$. For binary $r^\star$, the sign of $dP(\alpha)/d\alpha$ is determined by the sign of $\delta(\alpha)$. Hence, examining $\delta(\alpha)$ at the two extremes, $\alpha=1$ and $\alpha\to\infty$, provides useful information about the global behavior of $P(\alpha)$. The following theorem shows that $\delta(\infty)$ admits a closed-form expression.

\begin{theorem}[Limit of the self-reward gap]
\label{thm:asymptotic_gap}
As $\alpha\to\infty$, the self-reward gap $\Delta_s$ satisfies
\begin{equation}
\Delta_s
\triangleq
\lim_{\alpha\to\infty}\delta(\alpha)
=
s_C-s_I,
\label{eq:asymptotic_gap}
\end{equation}
where $s_C = \max_{y\in\mathcal{Y}_C} r_{\mathrm{self}}(y)$, and $s_I = \max_{y\in\mathcal{Y}_I} r_{\mathrm{self}}(y)$.
\end{theorem}
The signs of $\delta(1)$ and $\Delta_s$ reveal the following four cases.
If both are positive, increasing $\alpha$ improves performance near $\alpha=1$ and again for all sufficiently large $\alpha$, although $\delta(\alpha)$ may still change sign at intermediate values, causing $dP(\alpha)/d\alpha$ to change sign as well. If $\delta(1)>0$ but $\Delta_s<0$, increasing $\alpha$ initially improves performance but eventually becomes harmful. If $\delta(1)<0$ but $\Delta_s>0$, increasing $\alpha$ is initially harmful but eventually becomes beneficial. If both are negative, increasing $\alpha$ is harmful near $\alpha=1$ and again for sufficiently large $\alpha$. 

\subsection{
Extensions to Group-Relative RL
}
\label{subsec:rl_connection}

Going beyond test-time reasoning, another related sampling scenario is group-relative RL, where multiple responses to the same query are scored and compared to construct relative advantages for policy optimization, as in GRPO~\citep{shao2024deepseekmath}. We next show that replacing such standard rollouts with query-adaptive power sampling can also be beneficial. This follows from the same analytical framework with a modified objective, since RL requires informative reward variation among sampled responses~\citep{razin2026makes}, rather than simply maximizing correctness.

Precisely,
for a group of $K$ responses independently sampled from $\pi_\alpha(y\mid x)$ in group-relative RL, let $R_i$ denote the reward of the $i$-th response and
$\bar R=K^{-1}\sum_{i=1}^K R_i$. For binary correctness rewards, we define $C_K(\alpha)=\mathbb{E}\left[\frac{1}{K}\sum_{i=1}^K(R_i-\bar R)^2\right]$ denoting the expected within-group reward variation of $K$ sampled responses. Since $R_i$ is a Bernoulli random variable with success probability
$P(\alpha)$, $C_K(\alpha)$ can be rewritten as 
\begin{equation}
C_K(\alpha)
=
\frac{K-1}{K}P(\alpha)\bigl(1-P(\alpha)\bigr).
\label{eq:ck}
\end{equation}
A larger $C_K(\alpha)$ means more variation in the rewards within a group,
making the relative advantages among sampled responses easier to distinguish
and yielding a clearer signal for policy optimization. Differentiating \eqref{eq:ck} and applying
Corollary~\ref{cor:gap_decomposition} gives the following theorem.

\begin{theorem}
\label{theo3}
For the expected within-group reward variation $C_K(\alpha)$ defined above, we have
\begin{align}
\frac{d\mathcal{C}_{K}(\alpha)}{d\alpha}
&=
\frac{K-1}{K}
\left(1-2P(\alpha)\right)
\frac{dP(\alpha)}{d\alpha}
=
\frac{K-1}{K}
\left(1-2P(\alpha)\right)
P(\alpha)
\left(1-P(\alpha)\right)
\delta(\alpha).
\label{eq:dCK}
\end{align}
\end{theorem}
Theorem~\ref{theo3} reveals how adapting the sharpening exponent $\alpha$ affects correctness probability and within-group reward variation jointly. When $P(\alpha)<1/2$, an adaptation of $\alpha$ that increases $P(\alpha)$
also increases $C_K(\alpha)$. In contrast, once $P(\alpha)>1/2$, the same improvement in $P(\alpha)$ is accompanied by a decrease in $C_K(\alpha)$. According to \eqref{eq:ck}, $\mathcal{C}_{K}(\alpha)$ is maximized at $P(\alpha)=1/2$, which provides a simple criterion for adapting the sharpening strength in group-relative RL. We therefore adapt $\alpha$ to maximize $C_K(\alpha)$, or equivalently, to make the correctness probability close to $1/2$. The detailed procedure and experiments are presented in Appendix~\ref{sec:rl-math500}.

Lastly, we develop a unified framework that subsumes both the test-time objective $P(\alpha)$ and the within-group reward variation $C_K(\alpha)$; see Appendix~\ref{more results}. More broadly, this framework provides a recipe for adapting the sharpening strength in other power sampling-based settings. 
\section{Adaptive Power Sampling}
\label{sec:aps}

This section translates the theoretical insights into practical algorithms.

\subsection{Estimating the Sharpening Direction}

Theorem~\ref{thm:sharpening_derivative} and Corollary~\ref{cor:gap_decomposition} show that the sharpening exponent $\alpha$ can be adjusted to improve $P(\alpha)$ according to 
the alignment between true reward and model's self-reward.
In practice, however, the true task reward is unavailable at test time.
To address this challenge, we develop an efficient algorithm, \emph{Adaptive Power Sampling} (APS), which uses self-consistency among sampled answers as a surrogate signal to adaptively adjust $\alpha$ during inference.

APS proceeds iteratively. At iteration $t$, given the current exponent $\alpha_t$, suppose that $N_t \geq 2$ responses $\{y_{i,t}\}_{i=1}^{N_t}$ are available. For simplicity, we omit the iteration index $t$ and write $y_i$. For each response $y_i$, we compute its self-reward as
\begin{equation}
    s_i = r_{\mathrm{self}}(y_i) = \log p_\theta(y_i\mid x).
\end{equation}
To reduce variations in the scale of self-reward, we standardize the self-reward of each response as
\begin{equation} \label{definition of z_i}
z_i=\frac{s_i-\bar{s}_t}{\widehat{\sigma}_{s,t}} , 
\end{equation}
where $ \bar{s}_t = \frac{1}{N_t}\sum_{i=1}^{N_t}s_i$ and $\widehat{\sigma}_{s,t} = \left[ \frac{1}{N_t} \sum_{i=1}^{N_t}(s_i-\bar{s}_t)^2 \right]^{1/2}$.

Since the ground-truth correctness of each response is unavailable at test time, we design answer agreement as a surrogate signal. This is inspired by self-consistency methods \citep{wang2023self,aggarwal2023adaptiveconsistency} , where agreement across independently sampled reasoning
paths provides a useful indication of answer confidence and is often associated with answer correctness. 
By denoting $e_i$ as the normalized answer key of $y_i$, the
answer-agreement score of $y_i$ is defined as
\begin{equation} \label{definition of v_i}
v_i=\frac{\mathbf{1}[e_i\neq\mathrm{invalid}]}{N_t-1}
\sum_{j\neq i}\mathbf{1}[e_j=e_i],
\end{equation}
where $\mathbf{1}[\cdot]$ denotes the indicator function, which equals one when the condition inside the brackets is satisfied and zero otherwise. For a valid answer, $v_i$ measures the fraction of other sampled responses giving the same answer as $y_i$. For an invalid answer, $v_i=0$.
A larger $v_i$ indicates stronger support from the other samples and is treated as evidence that the answer $y_i$ is more likely to be correct.

Based on the answer-agreement scores and normalized self-rewards, we therefore define the following score for determining the sharpening direction, i.e., whether $\alpha$ should be increased or decreased:
\begin{align}
\widehat{g}_t
&=
\frac{1}{N_t}
\sum_{i=1}^{N_t}
(v_i-\bar{v}_t)z_i=
\frac{1}{N_t\widehat{\sigma}_{s,t}}
\sum_{i=1}^{N_t}
(v_i-\bar{v}_t)(s_i-\bar{s}_t).
\label{eq:sharpening_effect_surrogate}
\end{align}
where $\bar{v}_t =\frac{1}{N_t}\sum_{i=1}^{N_t}v_i$. The sign of $\hat{g}_t$ reflects whether answer agreement and self-reward vary in the same direction with $\hat{\sigma}_{s,t}>0$. In particular, $\hat{g}_t>0$ indicates that responses with stronger agreement tend to have higher self-reward, which is consistent with the favorable case identified in Section \ref{sec:aap}. We should therefore increase $\alpha$.
The magnitude of $\hat{g}_t$ further reflects the strength of this relationship and is used to determine the amount by which $\alpha$ is adjusted.
If $\hat{\sigma}_{s,t}=0$, the sampled responses have identical self-reward, and we set $\hat{g}_t=0$. When all agreement scores are identical, the samples provide no preferred sharpening direction; APS keeps $\alpha_t$ unchanged and collects additional samples. 

\subsection{Adaptive Exponent Update}
Given the estimate $\widehat{g}_t$, APS updates the sharpening exponent via gradient ascent strategy:
\begin{equation} \label{update of alpha}
{\alpha}_{t+1}
=
\operatorname{clip}
\left(
\alpha_t+\eta\widehat{g}_t,
\alpha_{\min},
\alpha_{\max}
\right),
\end{equation}
where $\eta$ is the stepsize and $\operatorname{clip}(\cdot)$ ensures that the updated exponent remains within 
the appropriate range $[\alpha_{\min},\alpha_{\max}]$. 
The exponent adaptation is performed under both sampling-quality and budget constraints. At iteration $t$, we monitor the batch-level MCMC acceptance rate $\widehat{q}_t=A_t/P_t$, where $A_t$ and $P_t$ denote the numbers of accepted and total Metropolis--Hastings (MH) proposals used to generate the current batch at $\alpha_t$, respectively. We use $\widehat{q}_t$ as a sampler-quality diagnostic~\citep{karan2025reasoning}.
If the optional quality threshold $q_{\min}$ is enabled and $\widehat{q}_t<q_{\min}$, the adaptive search is terminated. Let $C$ denote the total sample budget and $c$ the number of samples generated so far. The adaptive search continues only when the remaining budget is sufficient to generate another batch of $K$ samples. 
The adaptive search also terminates if either $|\widehat{g}_t|\leq\epsilon$, indicating insufficient evidence for changing $\alpha_t$, or $\alpha_{t+1}=\alpha_t$, indicating that the proposed update cannot move the exponent. 

APS retains the last exponent passing the sampling-quality check as $\alpha^\star$, with $\alpha^\star=1$ as the initial fallback. After adaptation, the remaining budget is spent at $\alpha^\star$, and samples collected at this exponent are pooled for majority voting over valid answers, breaking ties by the mean self-reward of supporting responses. Algorithm~\ref{alg:aps_core} summarizes the procedure. APS adjusts the sharpening strength for each query through the estimated alignment between answer agreement and self-reward: its sign determines whether to increase $\alpha$, while its magnitude determines the update size, without requiring ground-truth answers.

\begin{algorithm}[t]
\caption{Adaptive Power Sampling (APS)}
\label{alg:aps_core}
\begin{algorithmic}[1]
\REQUIRE Prompt $x$, base model $p_\theta$, $\alpha_1$, $\alpha_{\min}$, $\alpha_{\max}$, $K$, $C$, $\eta$, $q_{\min}$, $\epsilon$
\STATE Initialize $c\leftarrow0$, $\alpha^\star\leftarrow1$, and $\mathcal{B}_{\alpha}\leftarrow\emptyset$
\FOR{$t=1,2,\ldots$ while $c+K\leq C$}
    \STATE Draw $K$ sequence-level MCMC samples from $p_\theta(y|x)^{\alpha_t}$; set $c\leftarrow c+K$
    \STATE For each $y_i$, extract $e_i$, $s_i$, and display answer $d_i$
    \STATE Estimate sampler quality $\widehat{q}_t$ from the MCMC acceptance rate
    \IF{$q_{\min}$ is enabled and $\widehat{q}_t<q_{\min}$}
        \STATE \textbf{break}
    \ENDIF
    \STATE Add the current samples to $\mathcal{B}_{\alpha_t}$, set $\alpha^\star\leftarrow\alpha_t$ and $N_t\leftarrow|\mathcal{B}_{\alpha_t}|$, and compute $ v_i $ via \eqref{definition of v_i}

\IF{$v_i=1$ for all $i=1,\ldots,N_t$}
    \STATE \textbf{break}
\ELSIF{$v_i=\bar{v}_t$ for all $i=1,\ldots,N_t$}
    \STATE Set $\alpha_{t+1}\leftarrow\alpha_t$ and \textbf{continue}
\ENDIF
    \STATE Compute $z_i$ via \eqref{definition of z_i}; if $\hat{\sigma}_{s,t}=0$, set $z_i\leftarrow0$
    \STATE Estimate $ \widehat{g}_t$ via \eqref{eq:sharpening_effect_surrogate}, and set ${\alpha}_{t+1}$ via \eqref{update of alpha}
    \IF{$|\widehat{g}_t|\leq\epsilon$ or ${\alpha}_{t+1}=\alpha_t$ or $c+K>C$}
        \STATE \textbf{break}
    \ENDIF
\ENDFOR
\STATE Use the remaining budget to sample at $\alpha^\star$ and add the samples to $\mathcal{B}_{\alpha^\star}$; select the most frequent valid answer, breaking ties by the highest mean self-reward among its supporting samples
\RETURN The corresponding display answer, or an empty answer if all samples are invalid
\end{algorithmic}
\end{algorithm}

\section{Experiments}
\label{sec:experiments}

\subsection{Experimental Setup}
\paragraph{Models.} We evaluate APS on four models: Qwen2.5-7B~\citep{yang2024qwen25},
Qwen2.5-Math-7B~\citep{yang2024qwen25math},
DeepSeek-Math-7B~\citep{shao2024deepseekmath}, and
DeepSeek-Math-7B-RL~\citep{shao2024deepseekmath}.
Qwen2.5-7B is a general-purpose model, while Qwen2.5-Math-7B and DeepSeek-Math-7B are math-focused models.

\paragraph{Benchmarks and Baselines.} We use three reasoning benchmarks. MATH500~\citep{lightman2023lets}
evaluates multi-step mathematical reasoning through correctness of the final answer.
HumanEval~\citep{chen2021evaluating} contains 164 Python programming tasks
evaluated by functional correctness. GPQA Diamond~\citep{rein2024gpqa}
contains 198 expert-level multiple-choice questions in physics, chemistry,
and biology. For the compared methods, base decoding denotes standard
autoregressive sampling from the original model distribution, while
low-temperature applies token-level temperature scaling with
$T=1/\alpha$ \citep{wang2020contextual}. The sequence-level sampling baselines include several methods that use a fixed sharpening exponent $\alpha$ with different sampling strategies, such as power sampling \citep{karan2025reasoning}, scalable power sampling (SPS) \citep{ji2026scalable}, power sampling with sequential Monte Carlo (Power-SMC) \citep{azizi2026powersmc}, and auxiliary particle power sampling (APPS) \citep{nguyen2026future}. 
We set $\alpha=4$ for all the fixed-exponent baselines. More implementation details are provided in Appendix \ref{app:evaluation_details}.

\subsection{Main Results}

\paragraph{Pass@1 Performance.}
Table~\ref{tab:main_results} summarizes the pass@1 accuracy results on MATH500, HumanEval, and GPQA. GRPO (MATH) comes from \citet{shao2025spurious} for Qwen and DeepSeek-Math-7B-RL for DeepSeek \citet{shao2024deepseekmath}.
Compared with base decoding, low-temperature sampling generally yields performance improvements,
showing that sharpening the next-token distribution at each decoding step can already improve reasoning performance. The sequence-level sampling baselines with fixed sharpening exponent, including Power Sampling, SPS, Power-SMC, and APPS, generally outperform low-temperature sampling,
suggesting that sequence-level distribution sharpening can provide additional
gains beyond token-level temperature scaling. More importantly, APS achieves the best performance among all baselines in most cases. 
These results further validate the query-wise variation observed in Section \ref{motiexample}, where different queries favor different sharpening strengths. Unlike these fixed-exponent methods,
APS adjusts $\alpha$ according to the sampled responses and hence adapts the sharpening strength to the current query.  By adapting $\alpha$ online, APS could exploit the reasoning capability already contained in the base model more effectively.

We further compare APS with GRPO, which is used to post-train the model on MATH.
Although GRPO
achieves strong performance on the in-domain MATH500 benchmark, its gains are
less consistent on HumanEval and GPQA. In contrast, APS requires no additional
training and achieves consistently strong performance across all three
benchmarks. These results show that adaptive inference-time sharpening can
effectively improve reasoning performance across different tasks without
updating the model parameters or relying on an external reward for post-training.

\begin{figure}[H]
  \centering
  \includegraphics[width=0.98\linewidth]{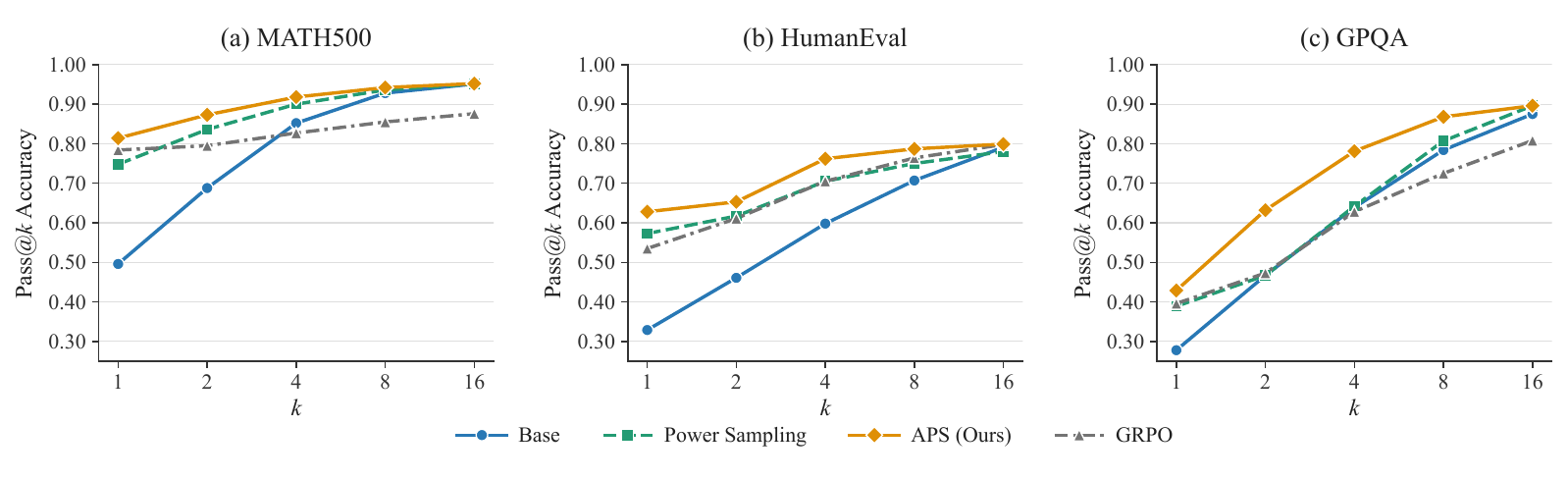}
  \caption{Pass@$k$ accuracy of Qwen2.5-Math-7B on MATH500, HumanEval, and GPQA.
  }
  \label{fig:passk_maths}
\end{figure}
\begin{table}[!htbp]
\caption{Pass@$1$ accuracy on MATH500 (M), HumanEval (H), and GPQA (G).}
\label{tab:main_results}
\centering
\normalsize
\setlength{\tabcolsep}{1.5pt}
\renewcommand{\arraystretch}{1.0}
\begin{tabular*}{\linewidth}{@{\extracolsep{\fill}}l*{9}{c}@{}}
\toprule
\textbf{Method}
& \multicolumn{3}{c}{\textbf{Qwen2.5-7B}}
& \multicolumn{3}{c}{\textbf{Qwen2.5-Math-7B}}
& \multicolumn{3}{c}{\textbf{DeepSeek-Math-7B}} \\
\cmidrule(lr){2-4}\cmidrule(lr){5-7}\cmidrule(l){8-10}
& \textbf{M} & \textbf{H} & \textbf{G}
& \textbf{M} & \textbf{H} & \textbf{G}
& \textbf{M} & \textbf{H} & \textbf{G} \\
\midrule
Base
& $0.498$ & $0.329$ & $0.278$
& $0.496$ & $0.329$ & $0.278$
& $0.362$ & $0.415$ & $0.333$ \\
Low-temperature
& $0.628$ & $0.524$ & $0.303$
& $0.690$ & $0.512$ & $0.353$
& $0.366$ & $0.427$ & $0.430$ \\
Power Sampling
& $0.706$ & $0.622$ & $0.318$
& $0.748$ & $0.573$ & $0.389$
& $0.424$ & $0.470$ & $0.345$ \\
SPS
& $0.708$ & $0.756$ & $0.349$
& $0.758$ & $0.604$ & $0.409$
& $0.464$ & $0.487$ & $0.364$ \\
Power-SMC
& $0.714$ & $0.761$ & $0.351$
& $0.762$ & $0.612$ & $0.413$
& $0.467$ & $0.498$ & $0.386$ \\
APPS
& $0.736$ & $0.610$ & $0.379$
& $0.782$ & $0.616$ & $0.429$
& $0.462$ & $0.512$ & $0.323$ \\
\textbf{APS (Ours)}
& $\underline{\textbf{0.734}}$ & $\underline{\textbf{0.732}}$ & $\underline{\textbf{0.369}}$
& $\underline{\textbf{0.814}}$ & $\underline{\textbf{0.628}}$ & $\underline{\textbf{0.429}}$
& $\underline{\textbf{0.472}}$ & $\underline{\textbf{0.518}}$ & $\underline{\textbf{0.389}}$ \\
\midrule
\emph{GRPO (MATH)}
& $\textbf{0.740}$ & $\textbf{0.561}$ & $\textbf{0.354}$
& $\textbf{0.785}$ & $\textbf{0.537}$ & $\textbf{0.399}$
& $\textbf{0.492}$ & $\textbf{0.524}$ & $\textbf{0.333}$ \\
\bottomrule
\end{tabular*}
\end{table}

\paragraph{APS on Post-Trained LLMs.}
Table~\ref{tab:aps_post_trained} evaluates whether APS remains effective when applied to an RL post-trained model.  On DeepSeek-Math-7B-RL, low-temperature sampling does not consistently improve the reasoning performance, while fixed-exponent power sampling yields only modest gains. In contrast,  APS consistently improves performance
across all three benchmarks, with particularly strong gains on MATH500 ($0.492\rightarrow0.549$) and HumanEval ($0.524\rightarrow0.598$). These results indicate that RL post-training does not eliminate query-level variation in the preferred sharpening strength, further demonstrating the merits of adaptive
rather than fixed power sampling.

\begin{table}[H]
\caption{Pass@$1$ accuracy of DeepSeek-Math-7B-RL on MATH500, HumanEval, and GPQA.}
\label{tab:aps_post_trained}
\centering
\normalsize
\setlength{\tabcolsep}{3.5pt}
\renewcommand{\arraystretch}{1.0}
\begin{tabular*}{0.95\linewidth}{@{\extracolsep{\fill}}lcccc@{}}
\toprule
\textbf{Benchmark} & \textbf{Base} & \textbf{Low-temperature}
& \textbf{Power Sampling} & \textbf{APS (ours)} \\
\midrule
MATH500   & $0.492$ & $0.412$ & $0.494$ & $\textbf{0.549}$ \\
HumanEval & $0.524$ & $0.524$ & $0.530$ & $\textbf{0.598}$ \\
GPQA      & $0.333$ & $0.303$ & $0.349$ & $\textbf{0.359}$ \\
\bottomrule
\end{tabular*}
\end{table}

\paragraph{Pass@$k$ Performance.}

To examine how performance scales with additional test-time samples, we evaluate Pass@$k$ for $k\in\{1,2,4,8,16\}$, where a problem is considered
solved if at least one of the $k$ independently sampled responses is correct. Figure~\ref{fig:passk_maths} presents the results on MATH500, HumanEval, and GPQA with Qwen2.5-Math-7B. Across all benchmarks, APS achieves strong Pass@$k$ performance, with its advantage being most pronounced for small $k$.  As $k$ increases, the performance gaps among the training-free sampling methods gradually narrow, since drawing more samples increases the chance of covering a correct response. Nevertheless, APS remains competitive across the full range of $k$, indicating that adaptive sharpening improves single-sample performance without substantially sacrificing multi-sample coverage. In contrast, GRPO exhibits relatively limited improvement as $k$ increases,
suggesting that post-training may concentrate the output distribution and reduce the benefit of drawing additional test-time samples. 
{Meanwhile, please refer to Appendix \ref{sec:rl-math500} and \ref{app:evaluation_details} for the results of adaptive power sharpening in group RL, MH-step sensitivity, effect of sample budgets, score distributions, and qualitative examples.}

\section{Conclusion}
\label{sec:conclusion}

This work proposes \emph{Adaptive Power Sampling} (APS) to equip power sampling for LLM reasoning with query-level adaptivity. Motivated by the empirical observation that the preferred sharpening strength is not universal across queries, we theoretically characterize how the effect of increasing the sharpening strength depends on the self-reward gap between correct and incorrect responses. We empirically show that this gap can be estimated using the correlation between answer agreement and the model's self-reward. Together, this theoretical characterization and practical adaptation enable APS to consistently improve test-time reasoning across diverse models and datasets, achieving consistent performance gain over power sampling with a fixed sharpening exponent, without additional training.

\subsection*{Reproducibility statement}
Proofs of the theoretical results are provided in
Appendices~\ref{app:proofs} and \ref{more results}.
The models, benchmarks, and experimental settings are described
in Section~\ref{sec:experiments}, with additional sampling and
evaluation details in Appendix~\ref{app:evaluation_details}.
An anonymized implementation of APS and the evaluation scripts
are included in the supplementary material.

\subsection*{AI use statement}

We used generative AI tools to assist with proofs, writing evaluation code, extracting illustrative cases from experimental outputs, polishing manuscript prose, and checking and summarizing cited literature. The authors verified the AI-assisted mathematical arguments, reviewed the code and text, checked the illustrative cases against the experimental outputs, and checked the literature. The authors take responsibility for the final content of this work.

\bibliographystyle{plainnat}
\bibliography{references_main4}

@inproceedings{karan2025reasoning,
  title     = {Reasoning with Sampling: Your Base Model is Smarter Than You Think},
  author    = {Karan, Aayush and Du, Yilun},
  booktitle = {International Conference on Learning Representations},
  year      = {2026}
}

@article{guo2026egps,
  title={Sample Where You Struggle: Sharpening Base Model Reasoning via Entropy-Guided Power Sampling},
  author={Guo, Hong and Guo, Nianhui and Meinel, Christoph and Yang, Haojin},
  journal={arXiv preprint arXiv:2606.09926},
  year={2026}
}

@article{ji2026scalable,
  title={Scalable Power Sampling: Unlocking Efficient, Training-Free Reasoning for {LLMs} via Distribution Sharpening},
  author={Ji, Xiaotong and Tutunov, Rasul and Zimmer, Matthieu and Ammar, Haitham Bou},
  journal={arXiv preprint arXiv:2601.21590},
  year={2026}
}

@article{azizi2026powersmc,
  title={Power-{SMC}: Low-Latency Sequence-Level Power Sampling for Training-Free {LLM} Reasoning},
  author={Azizi, Seyedarmin and Baghaei Potraghloo, Erfan and Ahmadi, Minoo and Kundu, Souvik and Pedram, Massoud},
  journal={arXiv preprint arXiv:2602.10273},
  year={2026}
}

@article{tomihari2026bridges,
  title={Power Distribution Bridges Sampling, Self-Reward {RL}, and Self-Distillation},
  author={Tomihari, Akiyoshi and Sato, Issei},
  journal={arXiv preprint arXiv:2605.04542},
  year={2026}
}

@inproceedings{he2025rewarding,
  title={Rewarding the Unlikely: Lifting {GRPO} Beyond Distribution Sharpening},
  author={He, Andre Wang and Fried, Daniel and Welleck, Sean},
  booktitle={Proceedings of the 2025 Conference on Empirical Methods in Natural Language Processing (EMNLP)},
  year={2025}
}

@article{brown2024large,
  title   = {Large Language Monkeys: Scaling Inference Compute with Repeated Sampling},
  author  = {Brown, Bradley and Juravsky, Jordan and Ehrlich, Ryan and Clark, Ronald and Le, Quoc V. and R{\'e}, Christopher and Mirhoseini, Azalia},
  journal = {arXiv preprint arXiv:2407.21787},
  year    = {2024}
}

@inproceedings{yao2023tree,
  title={Tree of Thoughts: Deliberate Problem Solving with Large Language Models},
  author={Yao, Shunyu and Yu, Dian and Zhao, Jeffrey and Shafran, Izhak and Griffiths, Tom and Cao, Yuan and Narasimhan, Karthik},
  booktitle={Advances in Neural Information Processing Systems},
  year={2023}
}

@article{snell2024scaling,
  title={Scaling {LLM} Test-Time Compute Optimally can be More Effective than Scaling Model Parameters},
  author={Snell, Charlie and Lee, Jaehoon and Xu, Kelvin and Kumar, Aviral},
  journal={arXiv preprint arXiv:2408.03314},
  year={2024}
}

@inproceedings{ouyang2022training,
  title     = {Training Language Models to Follow Instructions with Human Feedback},
  author    = {Ouyang, Long and Wu, Jeff and Jiang, Xu and Almeida, Diogo
               and Wainwright, Carroll L. and Mishkin, Pamela and Zhang, Chong
               and Agarwal, Sandhini and Slama, Katarina and Ray, Alex
               and Schulman, John and Hilton, Jacob and Kelton, Fraser
               and Miller, Luke and Simens, Maddie and Askell, Amanda
               and Welinder, Peter and Christiano, Paul and Leike, Jan
               and Lowe, Ryan},
  booktitle = {Advances in Neural Information Processing Systems},
  volume    = {35},
  pages     = {27730--27744},
  year      = {2022}
}

@article{shao2024deepseekmath,
  title   = {{DeepSeekMath}: Pushing the Limits of Mathematical Reasoning
             in Open Language Models},
  author  = {Shao, Zhihong and Wang, Peiyi and Zhu, Qihao and Xu, Runxin
             and Song, Junxiao and Bi, Xiao and Zhang, Haowei
             and Zhang, Mingchuan and Li, Y. K. and Wu, Y.
             and Guo, Daya},
  journal = {arXiv preprint arXiv:2402.03300},
  year    = {2024}
}

@article{guo2025deepseekr1,
  title   = {{DeepSeek-R1}: Incentivizing Reasoning Capability in {LLMs}
             via Reinforcement Learning},
  author  = {{DeepSeek-AI}},
  journal = {arXiv preprint arXiv:2501.12948},
  year    = {2025}
}

@inproceedings{yu2025dapo,
  title     = {{DAPO}: An Open-Source {LLM} Reinforcement Learning System at Scale},
  author    = {Yu, Qiying and Zhang, Zheng and Zhu, Ruofei and Yuan, Yufeng and
               Zuo, Xiaochen and Yu, Yue and Dai, Weinan and Fan, Tiantian and
               Liu, Gaohong and Liu, Juncai and Liu, Lingjun and Liu, Xin and
               Lin, Haibin and Lin, Zhiqi and Ma, Bole and Sheng, Guangming and
               Tong, Yuxuan and Zhang, Chi and Zhang, Mofan and Zhang, Ru and
               Zhang, Wang and Zhu, Hang and Zhu, Jinhua and Chen, Jiaze and
               Chen, Jiangjie and Wang, Chengyi and Yu, Hongli and Song, Yuxuan and
               Wei, Xiangpeng and Zhou, Hao and Liu, Jingjing and Ma, Wei-Ying and
               Zhang, Ya-Qin and Yan, Lin and Wu, Yonghui and Wang, Mingxuan},
  booktitle = {Advances in Neural Information Processing Systems},
  volume    = {38},
  year      = {2025}
}

@inproceedings{yue2025rlvr,
  title     = {Does Reinforcement Learning Really Incentivize Reasoning Capacity in LLMs Beyond the Base Model?},
  author    = {Yue, Yang and Chen, Zhiqi and Lu, Rui and Zhao, Andrew and Wang, Zhaokai and Yue, Yang and Song, Shiji and Huang, Gao},
  booktitle = {Advances in Neural Information Processing Systems},
  volume    = {38},
  year      = {2025}
}

@article{shao2025spurious,
  title   = {Spurious Rewards: Rethinking Training Signals in {RLVR}},
  author  = {Shao, Rulin and Li, Shuyue Stella and Xin, Rui
             and Geng, Scott and Wang, Yiping and Oh, Sewoong
             and Du, Simon Shaolei and Lambert, Nathan and Min, Sewon
             and Krishna, Ranjay and Tsvetkov, Yulia
             and Hajishirzi, Hannaneh and Koh, Pang Wei
             and Zettlemoyer, Luke},
  journal = {arXiv preprint arXiv:2506.10947},
  year    = {2025}
}

@inproceedings{huang2024sharpening,
  title     = {Self-Improvement in Language Models: The Sharpening Mechanism},
  author    = {Huang, Audrey and Block, Adam and Foster, Dylan J. and Rohatgi, Dhruv and Zhang, Cyril and Simchowitz, Max and Ash, Jordan T. and Krishnamurthy, Akshay},
  booktitle = {International Conference on Learning Representations},
  year      = {2025}
}

@article{song2025outcome,
  title   = {Outcome-Based Exploration for {LLM} Reasoning},
  author  = {Song, Yuda and Kempe, Julia and Munos, Remi},
  journal = {arXiv preprint arXiv:2509.06941},
  year    = {2025}
}

@inproceedings{wang2023self,
  author    = {Wang, Xuezhi and Wei, Jason and Schuurmans, Dale and Le, Quoc V. and Chi, Ed H. and Narang, Sharan and Chowdhery, Aakanksha and Zhou, Denny},
  title     = {Self-Consistency Improves Chain of Thought Reasoning in Language Models},
  booktitle = {International Conference on Learning Representations},
  year      = {2023}
}

@article{li2022competition,
  author  = {Li, Yujia and Choi, David and Chung, Junyoung and Kushman, Nate and Schrittwieser, Julian and Leblond, R{\'e}mi and Eccles, Tom and Keeling, James and Gimeno, Felix and Dal Lago, Agustin and Hubert, Thomas and Choy, Peter and de Masson d'Autume, Cyprien and Babuschkin, Igor and Chen, Xinyun and Huang, Po-Sen and Welbl, Johannes and Gowal, Sven and Cherepanov, Alexey and Molloy, James and Mankowitz, Daniel J. and Sutherland Robson, Esme and Kohli, Pushmeet and de Freitas, Nando and Kavukcuoglu, Koray and Vinyals, Oriol},
  title   = {Competition-Level Code Generation with {AlphaCode}},
  journal = {Science},
  volume  = {378},
  number  = {6624},
  pages   = {1092--1097},
  year    = {2022},
  doi     = {10.1126/science.abq1158}
}

@article{cobbe2021training,
  author  = {Cobbe, Karl and Kosaraju, Vineet and Bavarian, Mohammad and Chen, Mark and Jun, Heewoo and Kaiser, Lukasz and Plappert, Matthias and Tworek, Jerry and Hilton, Jacob and Nakano, Reiichiro and Hesse, Christopher and Schulman, John},
  title   = {Training Verifiers to Solve Math Word Problems},
  journal = {arXiv preprint arXiv:2110.14168},
  year    = {2021}
}

@inproceedings{lightman2023lets,
  author    = {Lightman, Hunter and Kosaraju, Vineet and Burda, Yura and Edwards, Harri and Baker, Bowen and Lee, Teddy and Leike, Jan and Schulman, John and Sutskever, Ilya and Cobbe, Karl},
  title     = {Let's Verify Step by Step},
  booktitle = {International Conference on Learning Representations},
  year      = {2024}
}

@inproceedings{chow2024inference,
  title     = {Inference-Aware Fine-Tuning for Best-of-N Sampling in Large Language Models},
  author    = {Chow, Yinlam and Tennenholtz, Guy and Gur, Izzeddin and Zhuang, Vincent and Dai, Bo and Kumar, Aviral and Agarwal, Rishabh and Thiagarajan, Sridhar and Boutilier, Craig and Faust, Aleksandra},
  booktitle = {International Conference on Learning Representations},
  year      = {2025}
}

@inproceedings{xie2023self,
  author    = {Xie, Yuxi and Kawaguchi, Kenji and Zhao, Yiran and Zhao, James Xu and Kan, Min-Yen and He, Junxian and Xie, Michael},
  title     = {Self-Evaluation Guided Beam Search for Reasoning},
  booktitle = {Advances in Neural Information Processing Systems},
  volume    = {36},
  year      = {2023}
}

@inproceedings{zhou2024language,
  author    = {Zhou, Andy and Yan, Kai and Shlapentokh-Rothman, Michal and Wang, Haohan and Wang, Yu-Xiong},
  title     = {Language Agent Tree Search Unifies Reasoning, Acting, and Planning in Language Models},
  booktitle = {Proceedings of the 41st International Conference on Machine Learning},
  volume    = {235},
  pages     = {62138--62160},
  year      = {2024}
}

@article{liu2025understanding,
  author  = {Liu, Zichen and Chen, Changyu and Li, Wenjun and Qi, Penghui and Pang, Tianyu and Du, Chao and Lee, Wee Sun and Lin, Min},
  title   = {Understanding {R1}-Zero-Like Training: A Critical Perspective},
  journal = {arXiv preprint arXiv:2503.20783},
  year    = {2025}
}

@inproceedings{lin2025cppo,
  title     = {CPPO: Accelerating the Training of Group Relative Policy Optimization-Based Reasoning Models},
  author    = {Lin, Zhihang and Lin, Mingbao and Xie, Yuan and Ji, Rongrong},
  booktitle = {Advances in Neural Information Processing Systems},
  volume    = {38},
  year      = {2025}
}

@article{zhang2025survey,
  title   = {A Survey on Test-Time Scaling in Large Language Models: What, How, Where, and How Well?},
  author  = {Zhang, Qiyuan and Lyu, Fuyuan and Sun, Zexu and Wang, Lei and
             Zhang, Weixu and Hua, Wenyue and Wu, Haolun and Guo, Zhihan and
             Wang, Yufei and Muennighoff, Niklas and King, Irwin and
             Liu, Xue and Ma, Chen},
  journal = {arXiv preprint arXiv:2503.24235},
  year    = {2025}
}

@article{hariri2026testtime,
  author  = {Hariri, Mohsen and Chen, Weicong and Shahini, Nahal and Singh, Vikash
             and Ye, Kai and Samandar, Amirhossein and Ganguly, Debargha
             and Sankar, Sreehari and Zhang, Yanyan and Wang, Shouren
             and Peng, Jerry and Zhang, Biyao and Hinczewski, Michael
             and Chaudhary, Vipin},
  title   = {Test-Time Scaling in Reasoning {LLMs}:
             Inference Regimes, Evaluation, and Reproducibility},
  journal = {arXiv preprint arXiv:2608.04001},
  year    = {2026}
}

@article{nguyen2026future,
  author  = {Nguyen, Tu and Zimmer, Matthieu and Tutunov, Rasul
             and Ji, Xiaotong and Bou Ammar, Haitham},
  title   = {The Model Knows, the Decoder Finds:
             Future Value Guided Particle Power Sampling},
  journal = {arXiv preprint arXiv:2605.02427},
  year    = {2026}
}

@inproceedings{wang2026cache,
  author    = {Wang, Ke and Yu, Zehao and Wang, Luwei and Huang, Yongchao},
  title     = {Cache Coherent Resampling for Efficient Test Time Scaling
               in {LLM} Reasoning via Adaptive Sequential Monte Carlo},
  booktitle = {Proceedings of the 43rd International Conference
               on Machine Learning},
  year      = {2026}
}

@article{zhou2026cutting,
  author  = {Zhou, Felix and Mehrotra, Anay and Liu, Quanquan C.},
  title   = {Reasoning with Sampling: Cutting at Decision Points},
  journal = {arXiv preprint arXiv:2605.30327},
  year    = {2026}
}

@article{meng2026degs,
  title={Depth-Entropy Guided Sampling for Training-Free {LLM} Reasoning},
  author={Meng, Zibin and Xie, Peng and Chen, Kani},
  journal={arXiv preprint arXiv:2607.09693},
  year={2026}
}

@article{chen2021evaluating,
  title   = {Evaluating Large Language Models Trained on Code},
  author  = {Chen, Mark and Tworek, Jerry and Jun, Heewoo and Yuan, Qiming and Pinto, Henrique Ponde de Oliveira and Kaplan, Jared and Edwards, Harri and Burda, Yuri and Joseph, Nicholas and Brockman, Greg and Ray, Alex and Puri, Raul and Krueger, Gretchen and Petrov, Michael and Khlaaf, Heidy and Sastry, Girish and Mishkin, Pamela and Chan, Brooke and Gray, Scott and Ryder, Nick and Pavlov, Mikhail and Power, Alethea and Kaiser, Lukasz and Bavarian, Mohammad and Winter, Clemens and Tillet, Philippe and Such, Felipe Petroski and Cummings, Dave and Plappert, Matthias and Chantzis, Fotios and Barnes, Elizabeth and Herbert-Voss, Ariel and Guss, William Hebgen and Nichol, Alex and Paino, Alex and Tezak, Nikolas and Tang, Jie and Babuschkin, Igor and Balaji, Suchir and Jain, Shantanu and Saunders, William and Hesse, Christopher and Carr, Andrew N. and Leike, Jan and Achiam, Josh and Misra, Vedant and Morikawa, Evan and Radford, Alec and Knight, Matthew and Brundage, Miles and Murati, Mira and Mayer, Katie and Welinder, Peter and McGrew, Bob and Amodei, Dario and McCandlish, Sam and Sutskever, Ilya and Zaremba, Wojciech},
  journal = {arXiv preprint arXiv:2107.03374},
  year    = {2021}
}

@inproceedings{rein2024gpqa,
  title     = {{GPQA}: A Graduate-Level Google-Proof {Q\&A} Benchmark},
  author    = {Rein, David and Hou, Betty Li and Stickland, Asa Cooper and Petty, Jackson and Pang, Richard Yuanzhe and Dirani, Julien and Michael, Julian and Bowman, Samuel R.},
  booktitle = {First Conference on Language Modeling},
  year      = {2024}
}

@article{yang2024qwen25,
  title   = {Qwen2.5 Technical Report},
  author  = {Yang, An and Yang, Baosong and Zhang, Beichen and Hui, Binyuan
             and Zheng, Bo and Yu, Bowen and Li, Chengyuan and Liu, Dayiheng
             and Huang, Fei and Wei, Haoran and others},
  journal = {arXiv preprint arXiv:2412.15115},
  year    = {2024}
}

@article{yang2024qwen25math,
  title   = {Qwen2.5-Math Technical Report: Toward Mathematical Expert Model
             via Self-Improvement},
  author  = {Yang, An and Zhang, Beichen and Hui, Binyuan and Gao, Bofei
             and Yu, Bowen and Li, Chengpeng and Liu, Dayiheng and Tu, Jianhong
             and Zhou, Jingren and Lin, Junyang and Lu, Keming and Xue, Mingfeng
             and Lin, Runji and Liu, Tianyu and Ren, Xingzhang and Zhang, Zhenru},
  journal = {arXiv preprint arXiv:2409.12122},
  year    = {2024}
}

@inproceedings{aggarwal2023adaptiveconsistency,
  title     = {Let's Sample Step by Step: Adaptive-Consistency for Efficient Reasoning and Coding with {LLM}s},
  author    = {Aggarwal, Pranjal and Madaan, Aman and Yang, Yiming and {Mausam}},
  booktitle = {Proceedings of the 2023 Conference on Empirical Methods in Natural Language Processing},
  pages     = {12375--12396},
  year      = {2023}
}

@inproceedings{hu2022lora,
  title={LoRA: Low-Rank Adaptation of Large Language Models},
  author={Hu, Edward J. and Shen, Yelong and Wallis, Phillip and Allen-Zhu, Zeyuan and Li, Yuanzhi and Wang, Shean and Wang, Lu and Chen, Weizhu},
  booktitle={International Conference on Learning Representations},
  year={2022}
}

@article{wang2020contextual,
  title={Contextual Temperature for Language Modeling},
  author={Wang, Pei-Hsin and Hsieh, Sheng-Iou and Chang, Shih-Chieh and Chen, Yu-Ting and Pan, Jia-Yu and Wei, Wei and Juan, Da-Chang},
  journal={arXiv preprint arXiv:2012.13575},
  year={2020}
}

@inproceedings{hendrycksmath2021,
  title={Measuring Mathematical Problem Solving With the MATH Dataset},
  author={Dan Hendrycks and Collin Burns and Saurav Kadavath and Akul Arora and Steven Basart and Eric Tang and Dawn Song and Jacob Steinhardt},
  booktitle={Proceedings of the Neural Information Processing Systems Track on Datasets and Benchmarks},
  volume={1},
  year={2021}
}

@inproceedings{razin2026makes,
  title={What Makes a Reward Model a Good Teacher? An Optimization Perspective},
  author={Razin, Noam and Wang, Zixuan and Strauss, Hubert and Wei, Stanley and Lee, Jason D. and Arora, Sanjeev},
  booktitle={Advances in Neural Information Processing Systems},
  volume={38},
  pages={59162--59222},
  year={2025}
}

\clearpage
\appendix

\section{Proofs}
\label{app:proofs}

\subsection{Proof of Theorem~\ref{thm:sharpening_derivative}}
For a fixed prompt $x$, according to the definition of $P(\alpha)$ in \eqref{eq:performance}, we have
\begin{align}
\frac{dP(\alpha)}{d\alpha}
&=
\frac{d}{d\alpha}
\sum_y
r^\star(y)
\frac{p_\theta(y\mid x)^\alpha}{Z_\alpha} \notag
\\
&=
\sum_y
r^\star(y)
\frac{
r_{\mathrm{self}}(y)p_\theta(y\mid x)^\alpha Z_\alpha
-
p_\theta(y\mid x)^\alpha Z_\alpha'
}{
Z_\alpha^2
} \notag
\\
&=
\sum_y
r^\star(y)\pi_\alpha(y)
\left[
r_{\mathrm{self}}(y)
-
\frac{Z_\alpha'}{Z_\alpha}
\right]  \notag
\\
&=
\mathbb{E}_{\pi_\alpha}
\left[
r^\star(Y)r_{\mathrm{self}}(Y)
\right]
-
\mathbb{E}_{\pi_\alpha}
\left[
r^\star(Y)
\right]
\mathbb{E}_{\pi_\alpha}
\left[
r_{\mathrm{self}}(Y)
\right]  \notag
\\
&=
\operatorname{Cov}_{\pi_\alpha}
\left(
r^\star(Y),r_{\mathrm{self}}(Y)
\right)
\label{eq:performance_covariance}
\end{align}
where $Z_\alpha' = \sum_y
p_\theta(y\mid x)^\alpha
\log p_\theta(y\mid x)$. The proof is thus completed.

\subsection{Proof of Corollary~\ref{cor:gap_decomposition}}

For binary $\rstar$, $P(\alpha)=\Pr_{\pip}(\rstar(Y)=1)$. We then have 
\begin{align}
\operatorname{Cov}_{\pi_\alpha}
\left(
r^\star(Y),r_{\mathrm{self}}(Y)
\right) 
&=
\mathbb{E}_{\pi_\alpha}
\left[
r^\star(Y)r_{\mathrm{self}}(Y)
\right]
-
\mathbb{E}_{\pi_\alpha}
\left[
r^\star(Y)
\right]
\mathbb{E}_{\pi_\alpha}
\left[
r_{\mathrm{self}}(Y)
\right]
\notag \\
&=
P(\alpha)\mu_C(\alpha)
-
P(\alpha)
\mathbb{E}_{\pi_\alpha}
\left[
r_{\mathrm{self}}(Y)
\right]
\notag \\
&=
P(\alpha)
\left[
\mu_C(\alpha)
-
P(\alpha)\mu_C(\alpha)
-
\bigl(1-P(\alpha)\bigr)\mu_I(\alpha)
\right]
\notag \\
&=
P(\alpha)\bigl(1-P(\alpha)\bigr)
\left[
\mu_C(\alpha)-\mu_I(\alpha)
\right] = P(\alpha)\bigl(1-P(\alpha)\bigr)\delta(\alpha),
\label{eq:binary_covariance}
\end{align}
where the last step uses $\E_{\pip}[\rself(Y)] = P(\alpha)\muC(\alpha) + (1{-}P(\alpha))\muI(\alpha)$.

\subsection{Proof of Lemma~\ref{lem:gap_dynamics}}
For fixed $x$, define 
\begin{equation}
Z_C(\alpha)
=
\sum_{y\in\mathcal{Y}_C}
p_\theta(y\mid x)^\alpha,
\qquad
Z_I(\alpha)
=
\sum_{y\in\mathcal{Y}_I}
p_\theta(y\mid x)^\alpha.
\label{eq:class_normalizers}
\end{equation}

For $y\in\mathcal{Y}_C$, the distribution of correct responses under
$\pi_\alpha$ is denoted as 
\begin{equation}
\Pr_{\pi_\alpha}\!\left(Y=y\mid r^\star(Y)=1\right)
=
\frac{p_\theta(y\mid x)^\alpha}{Z_C(\alpha)}.
\end{equation}
It follows that
\begin{align}
\mu_C(\alpha)
&=
\frac{
\sum_{y\in\mathcal{Y}_C}
p_\theta(y\mid x)^\alpha r_{\mathrm{self}}(y)
}{
Z_C(\alpha)
}
=
\frac{Z_C'(\alpha)}{Z_C(\alpha)}
=
\frac{d}{d\alpha}\log Z_C(\alpha).
\label{eq:mu_c_log_partition}
\end{align}

By differentiating \eqref{eq:mu_c_log_partition}, we obtain
\begin{align}
\frac{d\mu_C(\alpha)}{d\alpha}
&=
\frac{
\sum_{y\in\mathcal{Y}_C}
p_\theta(y\mid x)^\alpha r_{\mathrm{self}}^2(y)
}{
Z_C(\alpha)
}
-
\left(
\frac{
\sum_{y\in\mathcal{Y}_C}
p_\theta(y\mid x)^\alpha r_{\mathrm{self}}(y)
}{
Z_C(\alpha)
}
\right)^2
\nonumber\\
&=
\mathbb{E}_{\pi_\alpha}
\left[
r_{\mathrm{self}}^2(Y)
\mid r^\star(Y)=1
\right]
-
\mu_C^2(\alpha)
=
\sigma_C^2(\alpha).
\label{eq:mu_c_derivative}
\end{align}

The same calculation for the incorrect responses gives
\begin{equation}
\frac{d\mu_I(\alpha)}{d\alpha}
=
\sigma_I^2(\alpha).
\label{eq:mu_i_derivative}
\end{equation}

Since
$\delta(\alpha)=\mu_C(\alpha)-\mu_I(\alpha)$, we have
\begin{equation}
\frac{d\delta(\alpha)}{d\alpha}
=
\sigma_C^2(\alpha)-\sigma_I^2(\alpha),
\end{equation}
which proves the lemma.

\subsection{Proof of Theorem~\ref{thm:asymptotic_gap}}

Taking any $y,y'\in\mathcal{Y}_C$ with
$r_{\mathrm{self}}(y)<r_{\mathrm{self}}(y')$ as an example, as $\alpha\to\infty$, their conditional
probability ratio satisfies
\begin{equation}
\frac{
\Pr_{\pi_\alpha}(Y=y\mid r^\star(Y)=1)
}{
\Pr_{\pi_\alpha}(Y=y'\mid r^\star(Y)=1)
}
=
\exp\left\{
\alpha
\left[
r_{\mathrm{self}}(y)-r_{\mathrm{self}}(y')
\right]
\right\}
\rightarrow 0.
\end{equation}
Thus, the conditional distribution over $\mathcal{Y}_C$ concentrates on the highest-self-reward responses in the correct class. The same argument applies to $\mathcal{Y}_I$. Hence,
\begin{equation}
    \lim_{\alpha\to\infty}\muC(\alpha)=s_C,
    \qquad
    \lim_{\alpha\to\infty}\muI(\alpha)=s_I,
\end{equation}
and therefore
\begin{equation}
    \lim_{\alpha\to\infty}\deltasg(\alpha)=s_C-s_I=\Deltas,
\end{equation}
which proves the theorem.
\section{Additional theoretical results}
\label{more results}
\subsection{Second-Order Effect of Sharpening}

\begin{proposition}
\label{thm:curvature}
The second derivative of the performance with respect to the sharpening exponent $\alpha$ is given by
\begin{equation}
    \frac{d^2P}{d\alpha^2}
    =
    P(\alpha)(1-P(\alpha))
    \left[
    (1-2P(\alpha))\deltasg(\alpha)^2
    +
    \sigCsq(\alpha)-\sigIsq(\alpha)
    \right].
    \label{eq:curvature}
\end{equation}
\end{proposition}
Proposition \ref{thm:curvature} follows directly by differentiating the first-order performance relation in \eqref{eq:gap_decomposition} and applying Lemma~\ref{lem:gap_dynamics}. The sign of \eqref{eq:curvature} is determined by two factors. First, the term $(1-2P(\alpha))\deltasg^2$ could remain positive on hard queries with $P(\alpha)<1/2$ and negative on easy queries with $P(\alpha)>1/2$. From Lemma~\ref{lem:gap_dynamics}, the variance-difference term equals $d\delta(\alpha)/d\alpha$ and therefore indicates whether the self-reward gap is widening or narrowing as $\alpha$ increases. This decomposition immediately yields a simple sufficient condition for concavity. If $P(\alpha)>1/2$ and $\sigma_C^2(\alpha)\leq\sigma_I^2(\alpha)$, both terms in \eqref{eq:curvature} are non-positive, and hence $\frac{d^2P}{d\alpha^2}\leq 0$. Therefore, for queries that are already more likely to be solved than not under this variance condition,
the performance is locally concave in $\alpha$. If $\delta(\alpha)>0$ also holds, the performance still increases with $\alpha$, but at a decreasing rate.

\subsection{Generalization to Sample-Based Objectives}

We further generalize the preceding analysis to objectives defined over
multiple responses sampled from the power distribution. Let $s$ index a scenario and $Y_1,\ldots,Y_{K_s}$ be independently sampled from
$\pi_\alpha(y\mid x)$. Let $\ell_s(Y_{1:K_s})$ denote a statistic of interest with no explicit dependence on $\alpha$.
We define the corresponding scenario objective as
\begin{equation} \label{eq:scenario_objective}
    J_s(\alpha)
    =
    \mathbb{E}
    \left[
        \ell_s(Y_{1:K_s})
    \right],
\end{equation}
where the expectation is taken over the $K_s$ independent samples from
$\pi_\alpha(\cdot \mid x)$. Then the following theorem characterizes how the sharpening exponent $\alpha$ affects the scenario objective defined above.

\begin{proposition}[]
\label{thm:unified_sharpening}
The scenario objective in
\eqref{eq:scenario_objective} satisfies
\begin{equation}
    \frac{dJ_s(\alpha)}{d\alpha}
    =
    \operatorname{Cov}
    \left(
        \ell_s(Y_{1:K_s}),
        \sum_{i=1}^{K_s} r_{\mathrm{self}}(Y_i)
    \right),
    \label{eq:unified_sharpening_derivative}
\end{equation}
where the covariance is taken over the same $K_s$ independent responses
sampled from $\pi_\alpha(\cdot \mid x)$.
\end{proposition}

\begin{proof}
    Define the joint distribution of the $K_s$
independent responses under prompt $x$ as
\begin{equation}
    q_\alpha(y_{1:K_s}\mid x)
    =
    \prod_{i=1}^{K_s}
    \pi_\alpha(y_i\mid x).
\end{equation}

Since $\ell_s(Y_{1:K_s})$ has no explicit dependence on $\alpha$, applying
the log-derivative identity gives
\begin{align}
    \frac{dJ_s(\alpha)}{d\alpha}
    &=
    \mathbb{E}_{q_\alpha}
    \left[
        \ell_s(Y_{1:K_s})
        \frac{\partial}{\partial\alpha}
        \log q_\alpha(Y_{1:K_s}\mid x)
    \right].
\end{align}

Since the $K_s$ responses are independently sampled,
\begin{align}
    \frac{\partial}{\partial\alpha}
    \log q_\alpha(Y_{1:K_s}\mid x)
    &=
    \sum_{i=1}^{K_s}
    \frac{\partial}{\partial\alpha}
    \log \pi_\alpha(Y_i\mid x).
\end{align}
From the power distribution $\pi_\alpha(y\mid x)
    =
    \frac{p_\theta(y\mid x)^\alpha}{Z_\alpha(x)}$, we have
\begin{align}
    \frac{\partial}{\partial\alpha}
    \log \pi_\alpha(Y_i\mid x)
    &=
    \log p_\theta(Y_i\mid x)
    -
    \frac{\partial}{\partial\alpha}
    \log Z_\alpha(x) \notag  \\ 
    &=
    r_{\mathrm{self}}(Y_i)
    -
    \mathbb{E}_{\pi_\alpha}
    \left[
        r_{\mathrm{self}}(Y)
    \right].  
\end{align}
Therefore,
\begin{align}
    \frac{dJ_s(\alpha)}{d\alpha}
    &=
    \mathbb{E}
    \left[
        \ell_s(Y_{1:K_s})
        \left(
            \sum_{i=1}^{K_s} r_{\mathrm{self}}(Y_i)
            -
            K_s
            \mathbb{E}_{\pi_\alpha}
            \left[
                r_{\mathrm{self}}(Y)
            \right]
        \right)
    \right].
\end{align}

Since the responses are identically distributed,
\begin{equation}
    K_s
    \mathbb{E}_{\pi_\alpha}
    \left[
        r_{\mathrm{self}}(Y)
    \right]
    =
    \mathbb{E}
    \left[
        \sum_{i=1}^{K_s} r_{\mathrm{self}}(Y_i)
    \right].
\end{equation}
Hence,
\begin{align}
    \frac{dJ_s(\alpha)}{d\alpha}
    & \!=\!
    \mathbb{E} \!
    \left[ \!
        \ell_s(Y_{1:K_s})\!\!
        \left(
            \sum_{i=1}^{K_s} r_{\mathrm{self}}(Y_i)
            \!-\!
            \mathbb{E} \!
            \left[
                \sum_{i=1}^{K_s} r_{\mathrm{self}}(Y_i) \!
            \right]
        \!\right)\!
    \right] \!\!=\!\!
    \operatorname{Cov}\!
    \left( \!
        \ell_s(Y_{1:K_s}),
        \sum_{i=1}^{K_s} r_{\mathrm{self}}(Y_i) \!
    \right).
\end{align}
\end{proof}

Proposition~\ref{thm:unified_sharpening} shows that the marginal effect of sharpening is determined by the alignment between 
the scenario-specific statistic and the aggregate self-reward of the sampled responses. For 
test-time inference, setting $K_s=1$ and $\ell_s(Y)=r^\star(Y)$ recovers the performance derivative  derived in Section \ref{sec4.2}. For rollout generation in group-relative RL, taking $\ell_s$ as the
within-group reward variance recovers $C_K(\alpha)$. This provides a common way to analyze how sharpening affects both test-time
reasoning and RL rollout generation.

\section{
Adaptive Power Sharpening in group-relative RL
}
\label{sec:rl-math500}

\subsection{Adaptive Sharpening Exponent Selection}

In this appendix, we investigate whether the rollout objective derived in
Appendix~\ref{app:proofs}  can be used to adapt the sharpening strength for group-relative RL training, focusing on GRPO. For binary rewards, the analysis of Theorem \ref{theo3} shows that the sharpening exponent $\alpha$ directly controls the within-group reward variation. Motivated by this result, we adapt $\alpha$ separately for each training prompt before generating the rollout group used for the subsequent GRPO update.

To translate the rollout objective into a practical 
procedure, we compare candidate sharpening exponents using a small number of pilot responses. For each prompt and each exponent candidate
$\alpha^{(j)} \in \Delta$, we generate $k_p$ pilot responses using
token-level low-temperature $T=1/\alpha^{(j)}$ \citep{wang2020contextual}.
Since the correctness reward $r^\star$ is available during RL post-training,
we estimate the correctness probability for each candidate as
\begin{equation}
\widehat{P}_j
=
\frac{1}{k_p}
\sum_{i=1}^{k_p}
r^\star\!\left(\tilde{y}^{(j)}_i\right),
\end{equation}
where $\tilde{y}^{(j)}_i$ denotes the $i$-th pilot response associated
with candidate exponent $\alpha^{(j)}$. Based on \eqref{eq:ck}, the corresponding within-group reward variation is estimated by
\begin{equation}
\widehat{C}_K\!\left(\alpha^{(j)}\right)
=
\frac{K-1}{K}
\widehat{P}_j
\left(1-\widehat{P}_j\right).
\label{eq:empirical_ck}
\end{equation}
We then select
\begin{equation}
\alpha^\star
=
\arg\max_{\alpha^{(j)}\in\Delta}
\widehat{C}_K\!\left(\alpha^{(j)}\right).
\label{eq:argmax_ck}
\end{equation}
If there is a tie, we choose the smallest exponent as $\alpha^\star$. Note that token-level low-temperature is only used as a practical approximation to compare candidate exponents; After selecting $\alpha^\star$,
we draw $K$ sequence-level MCMC samples from
$p_\theta(y|x)^{\alpha^\star}$ using the same sampler as in Section~\ref{sec:aps}.
These samples form the rollout group, which is evaluated with correctness reward $r^\star$
and used for the standard GRPO update. Algorithm~\ref{alg:aps-argmax-ck} summarizes the complete rollout generation and update procedure.

\begin{algorithm}[H]
\caption{Power-sharpened Rollout Generation and GRPO update}
\label{alg:aps-argmax-ck}
\begin{algorithmic}[1]
\REQUIRE Prompt $x$, current model $p_\theta$,
$\Delta$, $k_p$, $K$, $r^\star$
\FOR{$\alpha^{(j)}\in\Delta$}
  \STATE Sample $k_p$ pilots at token temperature $T=1/\alpha^{(j)}$
  \STATE Compute $\widehat{\mathcal{C}}_K(\alpha^{(j)})$ using \eqref{eq:empirical_ck}
\ENDFOR
\STATE Select $\alpha^\star$ using \eqref{eq:argmax_ck}, breaking ties with
the smallest $\alpha$
\STATE Draw $K$ sequence-level MCMC samples from $p_\theta(y|x)^{\alpha^\star}$
\STATE Evaluate the rollouts with $r^\star$ and use them for the standard GRPO update
\end{algorithmic}
\end{algorithm}

\subsection{Experimental Results}

To examine the performance gain of the designed power-sharpened rollout, we train Low-Rank Adaptation (LoRA) \citep{hu2022lora} adapters on the Qwen2.5-7B~\citep{yang2024qwen25} using a MATH-derived reasoning dataset with 12,000 examples \citep{hendrycksmath2021}. Each training update contains
24 prompts with $K=8$ responses per prompt, resulting in 192 rollouts per
update. The LoRA rank and scaling factor are 32 and 64, respectively. We use
a learning rate of $3\times10^{-6}$, a GRPO KL coefficient of $\beta=0.04$,
and a clipping range of $0.2$. For APS-assisted rollout generation, we set
$\Delta=\{1,2,4\}$ and $k_p=3$, and select $\alpha^\star$ independently
for each prompt.
Each pilot generation uses a maximum length of 1,536 tokens. 

We compare APS-assisted GRPO with Vanilla GRPO, which follows the standard GRPO training procedure without adaptive exponent selection for rollout generation. We evaluate the trained models on MATH500~\citep{lightman2023lets} using Pass@1 accuracy with one sampled response per problem. The corresponding results are summarized in Table~\ref{tab:rl-math500}.

\begin{table}[H]
\centering
\small
\setlength{\tabcolsep}{5pt}
\begin{tabular*}{\linewidth}{@{\extracolsep{\fill}}lcccc@{}}
\toprule
Method & Updates & Relative time & Pass@$1$  \\
\midrule
Base, no GRPO & -- & -- & 49.0\% \\
Vanilla GRPO & 80 & $1.00\times$ & 52.0\%  \\
APS-assisted GRPO & 62 & $0.96\times$ & 54.2\%  \\
APS-assisted GRPO & 80 & $1.41\times$ & 57.6\%  \\
\bottomrule
\end{tabular*}
\caption{ Training cost and Pass@$1$ accuracy of Qwen2.5-7B on MATH500}
\label{tab:rl-math500}
\end{table}

As shown in Table~\ref{tab:rl-math500}, APS-assisted GRPO consistently
improves over Vanilla GRPO. Under a comparable wall-clock budget,
APS-assisted GRPO achieves 54.2\% Pass@1 after 62 steps, compared with
52.0\% for Vanilla GRPO, while using slightly less training time
($0.96\times$). With the same number of 80 updates, APS-assisted GRPO
further improves Pass@1 to 57.6\%, a gain of 5.6 percentage points.

\section{Supplementary Experimental Results}
\label{app:evaluation_details}

\subsection{Additional Evaluation Details}
\label{app:additional_eval}

All inference experiments are conducted on a node with 8 GPUs using vLLM v0.9.1. 

\paragraph{Prompts.}
We follow the prompt templates used in  \citep{karan2025reasoning,ji2026scalable}, as summarized in Table~\ref{tab:eval_prompts}.

\paragraph{Hyperparameter settings of APS.} 
APS uses a maximum generation length of $T_{\max}=3072$ tokens.  For sequence-level MCMC sampling, we follow the implementation settings of \citep{karan2025reasoning}. Specifically, we use a block size of $B=192$, which determines the number of tokens updated in each
proposal block, and perform $N_{\mathrm{MCMC}}=10$ MH steps at each stage. We initialize APS with initial exponent $\alpha_1=4$, batch size $K=4$, total sample budget $C=12$, step size $\eta=2$, tolerance $\epsilon=0.05$, exponent range $\alpha \in [1,8]$ with $\alpha_{\mathrm{min}}=1$ and $\alpha_{\mathrm{max}}=8$, optional quality threshold $q_{\min}=0.05$. Unless otherwise stated, we report results averaged over 3 random seeds.

\begin{table}[htbp]
\centering
\small
\setlength{\tabcolsep}{3pt}
\begin{tabular}{p{0.14\linewidth}p{0.40\linewidth}p{0.40\linewidth}}
\toprule
\textbf{Task} & \textbf{Qwen Prompt} & \textbf{DeepSeek Prompt} \\
\midrule
\textbf{MATH500} &
{\raggedright\ttfamily Can you solve the following math problem?\newline
Please reason step by step, and put your final answer within \textbackslash\textbackslash boxed\{\}\par} &
{\raggedright\ttfamily Please reason step by step, and put your final answer within \textbackslash\textbackslash boxed\{\}\par} \\
\midrule
\textbf{HumanEval} &
{\raggedright\ttfamily Write a Python function to solve the following problem:\par} &
{\raggedright\ttfamily Complete the following Python function:\par} \\
\midrule
\textbf{GPQA} &
{\raggedright\ttfamily Answer the following multiple-choice question. The last line of your response should be of the following format: '\textbackslash\textbackslash boxed\{LETTER\}' (without quotes) where LETTER is one of ABCD (ex. '\textbackslash\textbackslash boxed\{A\}'). Think step by step before answering.\newline
\{Question\} A) \{\} B) \{\} C) \{\} D) \{\}\par} &
{\raggedright\ttfamily Answer the following multiple-choice question. The last line of your response should be of the following format: '\textbackslash\textbackslash boxed\{LETTER\}' (without quotes) where LETTER is one of ABCD (ex. '\textbackslash\textbackslash boxed\{A\}'). Think step by step before answering.\newline
\{Question\} A) \{\} B) \{\} C) \{\} D) \{\}\par} \\
\bottomrule
\end{tabular}
\caption{Prompt templates used for baselines and our method. }

\label{tab:eval_prompts}
\end{table}

\paragraph{Hyperparameter settings of baselines.}  
For low-temperature sampling, we apply token-level temperature scaling with
temperature $\tau=1/\alpha=0.25$ with $\alpha=4$.

For power sampling, we follow \citep{karan2025reasoning} with
$\alpha=4$,  $T_{\max}=3072$, 
$B=192$, and $N_{\mathrm{MCMC}}=10$ MH transitions per stage. 

For SPS, we follow \citep{ji2026scalable} with $\alpha=4$, $T_{\max}=3072$, and
$B=192$. At each decoding step, SPS considers the top
$K_t=8$ candidate tokens and uses $M_t=8$ future rollouts for each
candidate.

For Power-SMC, we use the reported results from
\citep{azizi2026powersmc} and \citep{meng2026degs}, depending on
the corresponding model and benchmark. For APPS, we use the results reported
in \citep{nguyen2026future}. We refer to the original papers for their
implementation details and hyperparameter settings.

\subsection{Effect of MCMC Steps}
\label{app:mcmc_steps}

We vary the number of MH steps, $N_{\mathrm{MCMC}}$, and assess whether the default
setting is sufficient and how APS is affected by MCMC approximation quality. The evaluation on MATH500 uses Qwen2.5-7B,
Qwen2.5-Math-7B, and DeepSeek-Math-7B, with
$N_{\mathrm{MCMC}}\in\{6,8,10,12\}$ steps per block. Using the same sampling settings as in Appendix~\ref{app:additional_eval}, we compare APS
with fixed-exponent power sampling at $\alpha=4$. The Pass@1 accuracy of base decoding  in Table~\ref{tab:main_results} is include  as fixed references.

Figure~\ref{fig:mcmc_steps} shows that $N_{\mathrm{MCMC}}=10$ MH steps are sufficient under
the tested settings. To be specific, increasing $N_{\mathrm{MCMC}}$ from 10 to 12 changes
APS Pass@1 by $-0.6$, $+0.1$, and $-0.5$ percentage points on Qwen2.5-7B,
Qwen2.5-Math-7B, and DeepSeek-Math-7B, respectively. Additional MH steps beyond the default setting therefore provide no consistent accuracy gain.

APS also maintains its advantage when fewer MH steps are used, outperforming
power sampling at every evaluated step count for all three models. Even with only
6 steps, APS improves Pass@1 accuracy by $2.5$, $5.5$, and $6.2$ percentage points on the
three models, compared with $2.8$, $6.6$, and $4.8$ points at 10 steps.
The persistent advantage indicates that APS is robust to the tested range of MH steps.

\begin{figure}[htbp]
  \centering
  \includegraphics[width=\linewidth]{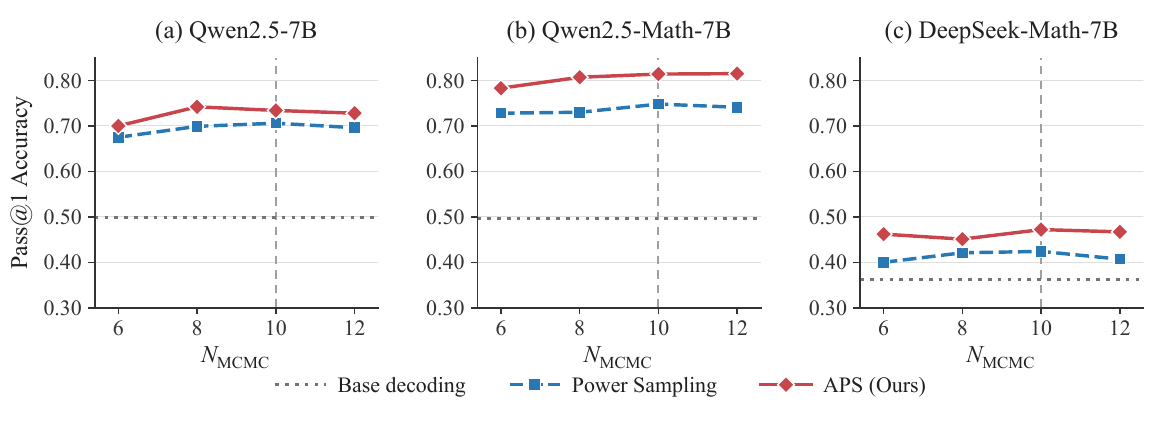}
  \caption{Pass@1 accuracy on MATH500 under different numbers of MH steps. The vertical dashed line marks the default $N_{\mathrm{MCMC}}=10$.
  }
  \label{fig:mcmc_steps}
\end{figure}

\subsection{Effect of Sample Budget}
\label{app:sample_budget}

We vary the total sample budget $C$ and evaluate APS with Qwen2.5-Math-7B on HumanEval with $C\in\{4,8,12,16\}$  while keeping the other settings in Appendix~\ref{app:additional_eval} fixe.  The Pass@1 
accuracies of base decoding and power sampling in Table~\ref{tab:main_results} are include  as fixed references.

Figure~\ref{fig:sample_budget} shows that the reasoning performance of APS generally improves as the total sample
budget $C$ increases. For example, increasing $C$ from 4 to 16 raises
Pass@1 accuracy from 0.585 to 0.646, a gain of 6.1 percentage points. This is because 
a larger budget allows APS to perform more rounds of sharpening exponent adaptation,
providing more opportunities to refine $\alpha$ and identify a sharpening
strength better suited to the current query.

\begin{figure}[H]
  \centering
  \includegraphics[width=0.6\linewidth]{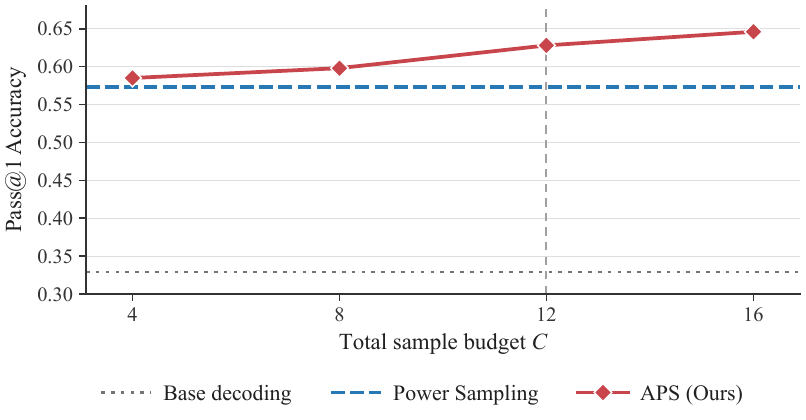}
  \caption{Pass@1 accuracy on HumanEval with Qwen2.5-Math-7B under different sample budgets. The vertical dashed line marks the default $C=12$.}
  \label{fig:sample_budget}
\end{figure}

\subsection{Log-Likelihood and Confidence Densities}
\label{app:score_densities}

We examine log-likelihood and confidence densities of responses generated by different sampling methods on MATH500 with Qwen2.5-Math-7B. For each response $y$, we compute its mean token log-likelihood $L(y)=T^{-1}\sum_{t=1}^{T}\log p_\theta(y^t\mid x,y^{<t})$ and use the mean negative predictive entropy $C(y)=-T^{-1}\sum_{t=1}^{T}H[p_\theta(\cdot\mid x,y^{<t})]$ as a confidence score \citep{karan2025reasoning}. Both quantities are evaluated under the base model. For each query, we collect one response from base decoding, one from fixed-exponent power sampling with $\alpha=4$, and one from APS,
and compute $L(y)$ and $C(y)$ for each response.

\

Figures~\ref{fig:loglik_density} and~\ref{fig:confidence_density} show the
empirical densities of mean token log-likelihood and negative predictive entropy. Relative to base decoding, fixed-exponent power sampling concentrates both distributions near zero, indicating a
preference for responses with higher base-model likelihood and confidence. APS preserves a similar high-density mode but exhibits broader distributions (standard deviation: $0.0201\!\rightarrow\!0.0494$ for $L(y)$ and
$0.0402\!\rightarrow\!0.1013$ for $C(y)$), with slightly lower mean scores
(mean: $-0.0329\!\rightarrow\!-0.0421$ for $L(y)$ and
$-0.0671\!\rightarrow\!-0.0858$ for $C(y)$).
These results suggest that APS retains the high-likelihood, high-confidence
region favored by power sampling while covering a broader range of responses
through query-adaptive sharpening.

\begin{figure}[H]
  \centering
  \includegraphics[width=0.65\linewidth]{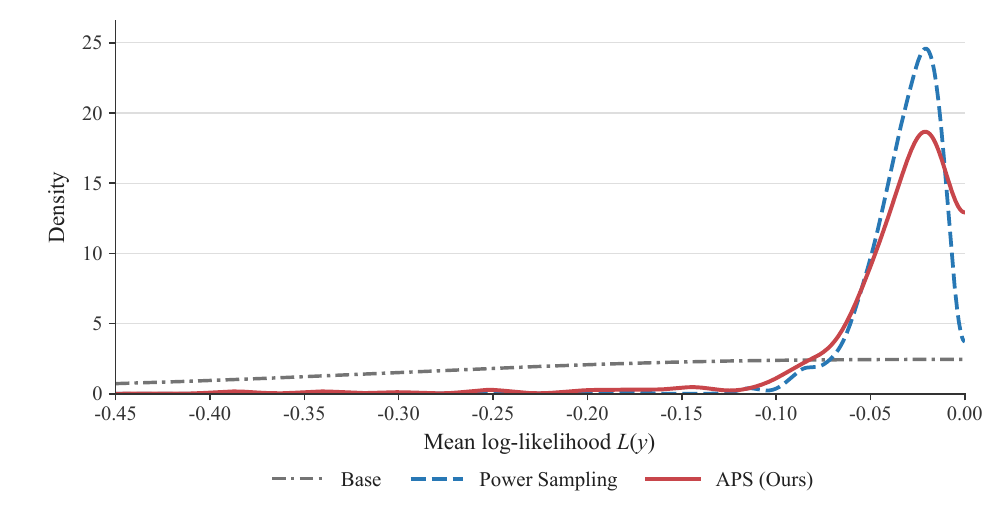}
  \caption{Density of mean token-level log-likelihood on MATH500 with
  Qwen2.5-Math-7B.}
  \label{fig:loglik_density}
\end{figure}

\begin{figure}[H]
  \centering
  \includegraphics[width=0.65\linewidth]{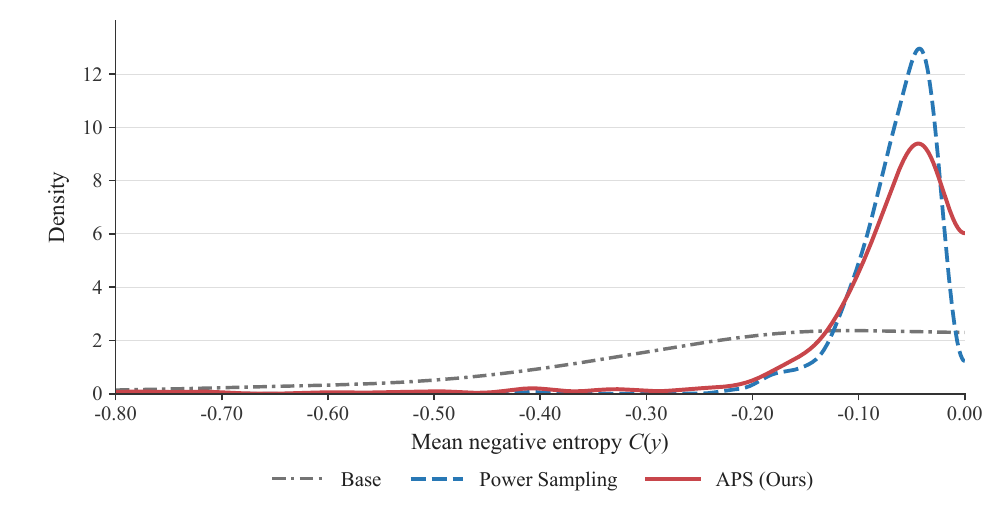}
  \caption{Density of mean negative predictive entropy on MATH500 with Qwen2.5-Math-7B. }
  \label{fig:confidence_density}
\end{figure}

\subsection{More Qualitative Examples}
\label{app:qualitative}
\raggedbottom

Tables~\ref{tab:qual:math500_pass1}--\ref{tab:qual:gpqa_pass1}
compare APS, power sampling, and GRPO on representative examples from the
three benchmarks considered in this paper. The first two examples use Qwen2.5-Math-7B, while the last example uses Qwen2.5-7B. Across the three examples, APS returns the correct output whenever fixed-exponent power sampling fails. In the MATH500 example, GRPO also succeeds, while APS
adapts $\alpha$ away from the fixed value and recovers the correct answer. In the HumanEval and GPQA examples, both power sampling and GRPO fail, whereas APS returns the correct output with a query-specific sharpening exponent.
Together, these cases illustrate the potential benefit of adapting the sharpening
strength to individual queries.

\begin{table}[H]
\centering
\small
\setlength{\tabcolsep}{5pt}
\renewcommand{\arraystretch}{1.15}
\begin{tabular}{@{}p{0.20\linewidth}p{0.60\linewidth}p{0.12\linewidth}@{}}
\toprule
\multicolumn{3}{@{}p{0.96\linewidth}@{}}{\itshape
The results of a cross-country team's training run are graphed
as distance against time. The graph contains Evelyn $(1.25,4.5)$,
Briana $(2.5,2.2)$, Carla $(4.25,5.2)$, Debra $(5.6,2.8)$, and
Angela $(6.8,1.4)$. Which student has the greatest average speed?
\newline Correct answer: $\boxed{\text{Evelyn}}$.} \\
\midrule
Method & Representative response (excerpt) & Result \\
\midrule

APS\newline
($\alpha^\star{=}3.605$)
&
The response computes distance divided by time for each student and reports
\texttt{('Evelyn', 3.6)}. It concludes:
``The student with the greatest average speed is Evelyn, with an average
speed of 3.6 units per time unit.''
Final answer: $\boxed{\text{Evelyn}}$.
&
True
\\

\midrule

Power Sampling\newline
($\alpha{=}4$)
&
The response extracts the same five points and writes the correct code
for the ratios and their maximum, but its generated output is
\texttt{('Carla', 1.2205882352941178)}.
It concludes:
``The student with the greatest average speed is Carla,
with an average speed of approximately 1.22 units per time unit.''
Final answer: $\boxed{\text{Carla}}$.
&
False
\\

\midrule

GRPO
&
The response computes the average speeds as Evelyn: $3.6$;
Briana: $0.88$; Carla: $1.22$; Debra: $0.5$; Angela: $0.206$.
It concludes:
``Based on the calculated average speeds, Evelyn has the greatest
average speed of 3.6 units.''
Final answer: $\boxed{\text{Evelyn}}$.
&
True
\\

\bottomrule
\end{tabular}

\caption{A representative MATH500 example with Qwen2.5-Math-7B. APS and GRPO return the
correct answer, while power sampling returns an incorrect
answer despite using the correct computation.}
\label{tab:qual:math500_pass1}
\end{table}

\begin{table}[H]
\centering
\small
\setlength{\tabcolsep}{5pt}
\renewcommand{\arraystretch}{1.15}
\begin{tabular}{@{}p{0.20\linewidth}p{0.60\linewidth}p{0.12\linewidth}@{}}
\toprule
\multicolumn{3}{@{}p{0.96\linewidth}@{}}{\itshape
 Complete
\texttt{intersperse(numbers: List[int], delimeter: int) -> List[int]}. Insert \texttt{delimeter} between every two consecutive elements of
\texttt{numbers}. \newline Expected outputs:
\texttt{intersperse([], 4) = []} and
\texttt{intersperse([1, 2, 3], 4) = [1, 4, 2, 4, 3]}.} \\
\midrule
Method & Completion & Passed \\
\midrule

APS\newline
($\alpha^\star{=}2.497$)
&
\begin{minipage}[t]{\linewidth}\small\ttfamily
result = []\\
for i in range(len(numbers)):\\
\hspace*{1.5em}result.append(numbers[i])\\
\hspace*{1.5em}if i < len(numbers) - 1:\\
\hspace*{3em}result.append(delimeter)\\
return result
\end{minipage}
&
True
\\

\midrule

Power Sampling\newline
($\alpha{=}4$)
&
\begin{minipage}[t]{\linewidth}\small\ttfamily
result = []\\
for i in range(len(numbers) - 1):\\
\hspace*{1.5em}result.append(numbers[i])\\
\hspace*{1.5em}result.append(delimeter)\\
result.append(numbers[-1])\\
return result
\end{minipage}
&
False
\\

\midrule

GRPO
&
\begin{minipage}[t]{\linewidth}\small\ttfamily
return [delimeter] * (len(numbers) - 1) + numbers
\end{minipage}
&
False
\\

\bottomrule
\end{tabular}

\caption{A representative HumanEval example with Qwen2.5-Math-7B.
APS passes the tests, while power sampling fails on the empty list and GRPO places the delimiters before the input elements rather than between them.}
\label{tab:qual:humaneval_pass1}
\end{table}

\begin{table}[H] 
\centering 
\small 
\setlength{\tabcolsep}{5pt} 
\renewcommand{\arraystretch}{1.15} 
\begin{tabular}{@{}p{0.20\linewidth}p{0.60\linewidth}p{0.12\linewidth}@{}} 
\toprule 

\multicolumn{3}{@{}p{0.96\linewidth}@{}}{\itshape  
``Oh, I know you,'' the ribonucleoprotein particle says to the nascent 
chain as they meet. ``Pause there for a minute. Let me show you in; you 
really need some sugar.'' 
``It seems somewhat rough. I guess this is goodbye; I need to be on my 
way'', the chain replies. Where did they meet, and where is the chain 
heading? 
A) cytosol to the extracellular space; B) Golgi to the mitochondrion; 
C) membrane to the nucleus; D) ribosome to the proteasome.\newline 
Correct answer: $\boxed{A}$.} \\ 

\midrule 
Method & Representative response (excerpt) & Result \\ 
\midrule 

APS\newline
($\alpha^\star{=}3.754$)
&
\begin{minipage}[t]{\linewidth}\small 
``A) cytosol to the extracellular space: This is a common pathway for 
proteins synthesized in the cytosol to be secreted. [\ldots] 
D) ribosome to the proteasome: This is a pathway for proteins that are 
targeted for degradation. [\ldots] 
Given the context and the dialogue, the most likely scenario is that 
the nascent chain is being processed and then secreted from the cell. 
Therefore, the correct answer is: $\boxed{A}$'' 
\end{minipage}
&
True
\\ 

\midrule 

Power Sampling\newline
($\alpha{=}4$)
&
\begin{minipage}[t]{\linewidth}\small 
``D) ribosome to the proteasome: This option involves the nascent chain 
moving from the ribosome to the proteasome. The dialogue suggests that 
the nascent chain is being guided into a specific location (the ribosome) 
and then leaving that location (the proteasome). 
[\ldots] Based on the analysis, the most likely answer is: $\boxed{D}$'' 
\end{minipage}
&
False
\\ 

\midrule 

GRPO
&
\begin{minipage}[t]{\linewidth}\small 
``D) Ribosome to the proteasome: This option fits because the nascent 
chain is being guided into the Golgi apparatus, where it will undergo 
glycosylation, and then it will be transported to the proteasome for 
degradation if it's not needed. 
[\ldots] Therefore, the correct answer is $\boxed{D}$.'' 
\end{minipage}
&
False
\\ 

\bottomrule 
\end{tabular} 

\caption{A representative GPQA example with Qwen2.5-7B. APS returns the
correct answer A, while power sampling and GRPO both return the wrong answer D.}
\label{tab:qual:gpqa_pass1} 

\end{table}

\end{document}